\documentclass{article}

\PassOptionsToPackage{numbers, compress}{natbib}
\usepackage[final]{neurips_2020}

\usepackage[utf8]{inputenc} % allow utf-8 input
\usepackage[T1]{fontenc}    % use 8-bit T1 fonts
\usepackage{hyperref}       % hyperlinks
\usepackage{url}            % simple URL typesetting
\usepackage{booktabs}       % professional-quality tables
\usepackage{amsfonts}       % blackboard math symbols
\usepackage{nicefrac}       % compact symbols for 1/2, etc.
\usepackage{microtype}      % microtypography

\usepackage{amsmath,amsfonts,bm}

\def\eqref#1{equation~\ref{#1}}
\def\1{\bm{1}}

\def\vzero{{\bm{0}}}
\def\vone{{\bm{1}}}

\def\vp{{\bm{p}}}
\def\vq{{\bm{q}}}

\def\vx{{\bm{x}}}
\def\vy{{\bm{y}}}
\def\vz{{\bm{z}}}

\def\mA{{\bm{A}}}
\def\mB{{\bm{B}}}

\def\mD{{\bm{D}}}

\def\mI{{\bm{I}}}

\def\mT{{\bm{T}}}

\DeclareMathAlphabet{\mathsfit}{\encodingdefault}{\sfdefault}{m}{sl}
\SetMathAlphabet{\mathsfit}{bold}{\encodingdefault}{\sfdefault}{bx}{n}

\newcommand{\normltwo}{L^2}

\usepackage{graphicx}
\usepackage{caption}
\usepackage{subfigure}
\usepackage{wrapfig}

\usepackage{amssymb}
\usepackage{amsthm}
\usepackage{comment}
\usepackage[algo2e,ruled,vlined, linesnumbered]{algorithm2e}
\newtheorem{proposition}{Proposition}

\newcommand{\sect}[1]{Sec.~\ref{#1}}
\newcommand{\fig}[1]{Fig.~\ref{#1}}

\newcommand{\repo}{\url{github.com/facebookresearch/alebo}}

\title{Re-Examining Linear Embeddings for High-Dimensional Bayesian Optimization}

\author{%
  Benjamin Letham\\
  Facebook\\
  Menlo Park, CA \\
  \texttt{bletham@fb.com} \\
  \And
  Roberto Calandra\\
  Facebook AI Research\\
  Menlo Park, CA \\
  \texttt{rcalandra@fb.com} \\
  \And
  Akshara Rai\\
  Facebook AI Research\\
  Menlo Park, CA \\
  \texttt{akshararai@fb.com} \\
  \And
  Eytan Bakshy\\
  Facebook\\
  Menlo Park, CA \\
  \texttt{ebakshy@fb.com}
}

\begin{document}

\maketitle

\begin{abstract}
Bayesian optimization (BO) is a popular approach to optimize expensive-to-evaluate black-box functions. A significant challenge in BO is to scale to high-dimensional parameter spaces while retaining sample efficiency. A solution considered in existing literature is to embed the high-dimensional space in a lower-dimensional manifold, often via a random linear embedding. In this paper, we identify several crucial issues and misconceptions about the use of linear embeddings for BO. We study the properties of linear embeddings from the literature and show that some of the design choices in current approaches adversely impact their performance. We show empirically that properly addressing these issues significantly improves the efficacy of linear embeddings for BO on a range of problems, including learning a gait policy for robot locomotion.
\end{abstract}

\section{Introduction}
\label{sec:intro}
Bayesian optimization (BO) is a robust, sample-efficient technique for optimizing expensive-to-evaluate black-box functions~\citep{Mockus1989,Jones2001taxonomy}. 
BO has been successfully applied to diverse applications, ranging from automated machine learning~\citep{snoek2012practical,hutter2011smac} to robotics~\citep{Lizotte2007Automatic,Calandra2015a,Rai2018Bayesian}. 
One of the most active topics of research in BO is how to extend current methods to higher-dimensional spaces.
A common framework to tackle this problem is to consider a high-dimensional BO (HDBO) task as a standard BO problem in a low-dimensional embedding, where the embedding can be either linear (typically a random projection) or nonlinear (\textit{e.g.}, via a multi-layer neural network); see \sect{sec:related} for a full review. An advantage of this framework is that it explicitly decouples the problem of finding low-dimensional representations suitable for optimization from the actual optimization technique.

In this paper we study the use of linear embeddings for HDBO, and in particular we re-examine prior efforts to use random linear projections. Random projections are attractive because, by the Johnson-Lindenstrauss lemma, they can be approximately distance-preserving~\citep{johnson1984extensions} without requiring data to learn the embedding.
Random embeddings come with strong theoretical guarantees, but have shown mixed empirical performance for HDBO. Our goal here is not just to present a new HDBO method, but rather to improve understanding of important considerations for BO in an embedding.

The contributions of this paper are: 
\textbf{1)} We provide new results that identify why linear embeddings have performed poorly in HDBO. We show that existing approaches can produce representations that are not well-modeled by a Gaussian process (GP), or do not contain an optimum (\sect{sec:rembo}). 
\textbf{2)} We construct a representation with better properties for BO (\sect{sec:method}): modelability is improved with a Mahalanobis kernel tailored for linear embeddings and by adding polytope bounds to the embedding, and we show how to maintain a high probability that the embedding contains an optimum. 
\textbf{3)} We combine these improvements to form a new linear-embedding HDBO method, ALEBO, and show empirically that it outperforms a wide range of HDBO techniques, including on test functions up to $D$=1000, with black-box constraints, and for gait optimization of a multi-legged robot (Secs. \ref{sec:experiments} and \ref{sec:realworld}). These results show empirically that we have identified several important elements impacting the BO performance of linear embedding methods. Code to reproduce the results of this paper is available at \repo.

\section{Related Work}
\label{sec:related}
There are generally two approaches to extending BO into high dimensions. The first is to produce a low-dimensional embedding, do standard BO in this low-dimensional space, and then project up to the original space for function evaluations.
The foundational work on embeddings for BO is REMBO~\citep{wang2016rembo}, which uses a random projection matrix. 
\sect{sec:background} provides a thorough description of REMBO and several subsequent approaches based on random linear embeddings~\citep{qian16, nayebi19hesbo, binois18}. If derivatives of $f$ are available, the active subspace method can be used to recover a linear embedding~\citep{constantine14, eriksson18}, or approximate gradients can be used \citep{Djolonga13}. A linear embedding can also be learned end-to-end with the GP \citep{garnett2014linear}, however this requires estimating $D \times d$ parameters, which is infeasible in settings where the total number of evaluations is much less than $D$. BO can also be done in nonlinear embeddings through VAEs~\citep{gomez2018chemical, lu2018vae, moriconi2019hdbo}.  An attractive aspect of random embeddings is that they can be extremely sample-efficient, since the only model to be estimated is a low-dimensional GP.

The second approach to extend BO to high dimensions is to make use of surrogate models that better handle high dimensions, typically by imposing additional structure on the problem. Work along these lines include GPs with an additive kernel \citep{kandasamy2015high, wang2017batched,gardner2017discovering,wang2018batched,rolland18, mutny18}, cylindrical kernels \citep{oh18bock}, or deep neural network kernels \citep{antonova2017deep}. Random forests are used as the surrogate model in SMAC~\citep{hutter2011smac}.

Here, we focus on the embedding approach, and in particular the use of linear embeddings for HDBO. While REMBO can perform well in some HDBO tasks, subsequent papers have found it can perform poorly even on synthetic tasks with a true low-dimensional linear subspace \citep[\textit{e.g.,}][]{nayebi19hesbo}. In this paper, we analyze the properties of linear embeddings as they relate to BO, and show how to improve the representation of the function we seek to optimize.

\section{Problem Framework and REMBO}
\label{sec:background}
In this section, we define the problem framework and notation, and then describe REMBO, along with known challenges and follow-up work that has been proposed to address those issues.

\paragraph{Bayesian Optimization}
We consider the problem $\min_{\vx \in \mathcal{B}} f(\vx)$ where $f$ is a black-box function and $\mathcal{B}$ are box bounds. We assume gradients of $f$ are unavailable. The box bounds on $\vx$ specify the range of values that are reasonable or physically possible to evaluate.
For instance, \citep{gramacy16} used BO for an environmental remediation problem in which each $x_i$ represents a pumping rate of a particular pump, which has physical limitations.
The problem may also include nonlinear constraints $c_j(\vx) \leq 0$ where each $c_j$ is itself a black-box function.
BO is a form of sequential model-based optimization, where we fit a surrogate model for $f$ that is used to identify which parameters~$\vx$ should be evaluated next. The surrogate model is typically a GP, $f \sim \mathcal{GP}(m(\cdot), k(\cdot, \cdot))$, with mean function $m(\cdot)$ and a kernel $k(\cdot, \cdot)$. Under the GP prior, the posterior for the value of $f(\vx)$ at any point in the space is a normal distribution with closed-form mean and variance. Using that posterior, we construct an acquisition function $\alpha(\vx)$ that specifies the utility of evaluating $f$ at $\vx$, such as Expected Improvement (EI) \citep{jones98}. We find $\vx^* \in \arg\max_{\vx \in \mathcal{B}} \alpha(\vx)$, and in the next iteration evaluate $f(\vx^*)$.

GPs are useful for BO because they provide a well-calibrated posterior in closed form. With many kernels and acquisition functions, $\alpha(\vx)$ is differentiable and can be efficiently optimized. However, typical kernels like the ARD RBF kernel have significant limitations. GPs are known to predict poorly for dimension $D$ larger than 15--20 \citep{wang2016rembo,li16, nayebi19hesbo}, which prevents the use of standard BO in high dimensions. In HDBO, the objective $f : \mathbb{R}^D \rightarrow \mathbb{R}$ operates in a high-dimensional ($D$) space, which we call the \textit{ambient space}. When using linear embeddings for HDBO, we assume there exists a low-dimensional linear subspace that captures all of the variation of $f$. Specifically, let $f_d : \mathbb{R}^d \rightarrow \mathbb{R}$, $d \ll D$, and let $\mT \in \mathbb{R}^{d \times D}$ be a projection from $D$ down to $d$ dimensions.
The linear embedding assumption is that $f(\vx) = f_d(\mT \vx) \enspace \forall \vx \in \mathbb{R}^D$. $\mT$ is unknown, and we only have access to $f$, not $f_d$. We assume, without any loss of generality, that the box bounds are $\mathcal{B} = [-1, 1]^D$; the ambient space can always be scaled to these bounds.

\paragraph{Bayesian Optimization via Random Embeddings}

REMBO~\citep{wang2016rembo} specifies a $d_e$-dimensional embedding via a random projection matrix $\mA \in \mathbb{R}^{D \times d_e}$ with each element i.i.d. $\mathcal{N}(0, 1)$. BO is done in the embedding to identify a point $\vy \in \mathbb{R}^{d_e}$ to be evaluated, which is given objective value $f(\mA \vy)$. Without box bounds, REMBO comes with a strong guarantee: if $d_e \geq d$, then with probability 1 the embedding contains an optimum \citep[Thm. 2]{wang2016rembo}. Unfortunately, things become complicated when there are box bounds (or any other sort of bound) in the ambient space. One may select a point $\vy$ in the embedding to be evaluated and find that its projection to the ambient space, $\mA \vy$, falls outside $\mathcal{B}$. The embedding subspace is guaranteed to contain an optimum to the box-bounded problem \citep[Thm. 3]{wang2016rembo}, but that optimum is \textit{not} guaranteed to project up inside $\mathcal{B}$. When function evaluations are restricted to the box bounds, there is no guarantee that we can find an optimum in the embedding.

REMBO introduces three heuristics for handling box bounds. First, the embedding is given box bounds $[-\sqrt{d_e}, \sqrt{d_e}]^{d_e}$.
Second, if a point $\vy$ in the embedding projects up outside $\mathcal{B}$, then it is clipped to $\mathcal{B}$. Let $p_{\mathcal{B}} : \mathbb{R}^D \rightarrow \mathbb{R}^D$ be the $\normltwo$ projection that maps $\vx$ to its nearest point in $\mathcal{B}$. Then, $\vy$ is given value $f(p_{\mathcal{B}}(\mA \vy))$, which can always be evaluated. Note that clipping to $\mathcal{B}$ renders the projection of $\vy$ a nonlinear transformation whenever $\mA \vy \notin \mathcal{B}$. Third, the optimization is done with $k$=4 separate projections, to improve the chances of generating an embedding that contains an optimum inside $[-\sqrt{d_e}, \sqrt{d_e}]^{d_e}$. No data can be shared across the embeddings, which reduces sample efficiency.

\paragraph{Extensions of REMBO}
\citep{binois15} considers the issue of non-injectivity, where the $\normltwo$ projection causes many points in the embedding to map to the same vertex of $\mathcal{B}$. They introduce REMBO-$\phi k_\Psi$, which uses a warped kernel to reduce non-injectivity.
\citep{binois18} defines a projection matrix $\mB \in \mathbb{R}^{d \times D}$ that maps from the ambient space down to the embedding, and replaces the $\normltwo$ projection entirely with a projection $\gamma$ that maps $\vy$ to the closest point in $\mathcal{B}$ that satisfies $\mB \vx = \vy$. The $\gamma$ projection eliminates the need for heuristic box bounds on the embedding, while mapping to the same set of points as the $\normltwo$ projection. This method is called REMBO-$\gamma k_\Psi$.
\citep{binois2015thesis} studies the projection matrix and shows that BO performance can be improved for small $d$  by sampling each row of $\mA$ from the unit hypersphere $\mathbb{S}^{d_e-1}$.
If $\vz \sim \mathcal{N}(\vzero, \mI_{d_e})$, then $\frac{\vz}{||\vz||}$ is a sample from $\mathbb{S}^{d_e-1}$, so this amounts to normalizing the rows of the usual REMBO projection matrix. 
HeSBO \citep{nayebi19hesbo} avoids clipping to $\mathcal{B}$ by changing the projection matrix $\mA$ so that each row of $\mA$ has a single non-zero element in a random column, which is randomly set to $\pm 1$. In the ambient space, $x_i = \pm y_j$, where $j \sim \textrm{Unif}\{1, d_e\}$, $\pm$ is chosen uniformly at random, and $\vy \in [-1, 1]^{d_e}$.

\section{Challenges with Linear Embeddings}
\label{sec:rembo}
\label{sec:remboissues}

The heuristics just described introduce several issues that impact HDBO performance of linear embedding methods.  We highlight one recent observation from~\citep{binois18}, that most points in the embedding project up outside $\mathcal{B}$, and discuss three novel observations on why existing methods can struggle to learn useful high-dimensional surrogates.

\paragraph{Projection to the facets of $\mathcal{B}$ produces a nonlinear distortion in the function.} 
The function value at any point in the embedding is evaluated as $f(p_{\mathcal{B}}(\mA \vy))$. For points $\vy$ that project up outside of $\mathcal{B}$, this will be a nonlinear mapping despite the use of a linear embedding. This has a powerful, detrimental effect on the ability to model $f$ in the embedding.
\fig{fig:rembo_illustration} provides visualizations of an actual REMBO embedding for two classic test functions, both extended to $D$=100 by adding unused variables. The embedding for the Branin function contains all three optima, however there is clear, nonlinear distortion to the function caused by the clipping to $\mathcal{B}$. The embedding for the Hartmann6 function is even more heavily distorted.
\begin{figure}
    \centering
    \includegraphics{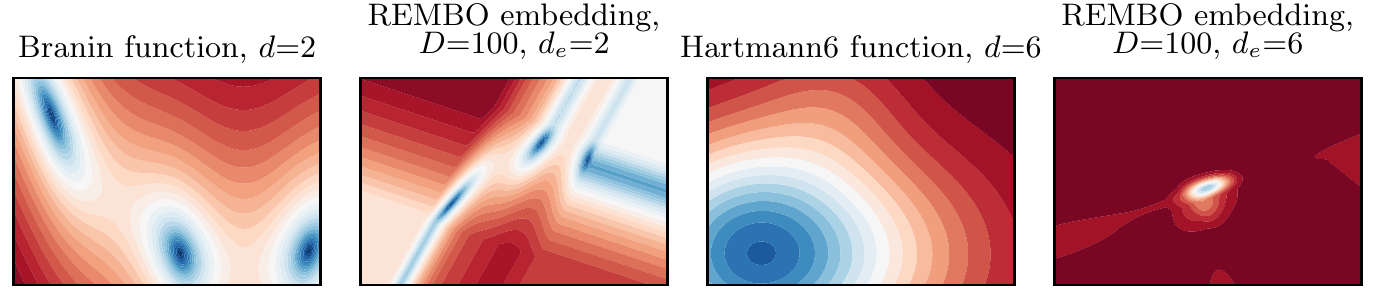}
    %\vspace{-10pt}
    \caption{
    A visualization of REMBO embeddings for two test functions. \textit{(Far left)} The Branin function, $d$=2, extended to $D$=100. \textit{(Center left)} A REMBO embedding of the $D$=100 Branin function. \textit{(Center right)} A center slice of the $d$=6 Hartmann6 function, similarly extended to $D$=100. \textit{(Far right)} The same slice of a REMBO embedding of that function. The embedding produces distortions and non-stationarity in the function that render it difficult to model.
    }
    \label{fig:rembo_illustration}
    %\vspace{-5pt}
\end{figure}
The distortion induced by clipping to a facet depends on the relative angles of the facet and the true embedding. Projection to a facet essentially induces a non-stationarity in the kernel: each of the $2D$ facets sits at different angles to the true subspace, and so the change in the rate of function variance will differ for each. To correct for the non-stationarity, we would have to estimate the true subspace $\mT$, which with $d \times D$ entries is not feasible for $D$ large.

The appeal of embeddings for HDBO is that they enable the use of standard BO on the embedding. However, these results show that for the REMBO projection with box bounds we may not be able to model in the embedding with a standard (stationary) GP, even if the function is well-modeled in the true low-dimensional space. The problem is especially acute for $d_e>2$ where, as we will see next, nearly all points in the embedding map to one of the $2D$ facets.

\begin{wrapfigure}{r}{0.37\linewidth}
    \centering
    \vspace{-10pt}
    \includegraphics{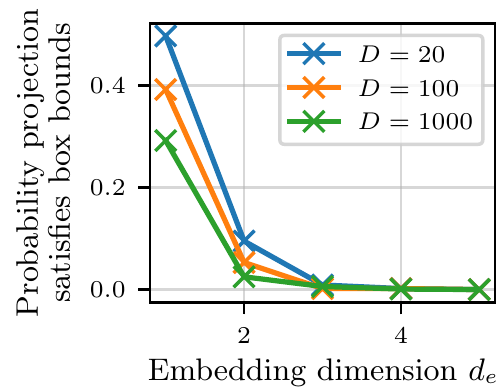}
%\vspace{-10pt}
\caption{The probability a randomly selected point in the REMBO embedding satisfies the ambient box bounds after being projected up. For $d_e>2$, nearly all points in the embedding project outside the bounds.}
\vspace{-20pt}
\label{fig:rembo_p_interior}
\end{wrapfigure}
\paragraph{Most points in the embedding map to the facets of~$\mathcal{B}$.} 
\fig{fig:rembo_p_interior} shows the probability that an interior point in the embedding projects up to the interior of~$\mathcal{B}$, measured empirically (with 1000 samples) by sampling $\vy$ uniformly from $[-\sqrt{d_e}, \sqrt{d_e}]^{d_e}$, sampling $\mA$ with $\mathcal{N}(0, 1)$ entries, and then checking if $\mA \vy \in \mathcal{B}$ . Even for small $D$, with $d_e>2$ practically all of the volume in the embedding projects up outside the box bounds, and is thus clipped to a facet of~$\mathcal{B}$.
This is an issue because it means the optimization will be done primarily on the facets of $\mathcal{B}$. We saw in \fig{fig:rembo_illustration} that the function behaves very differently on points projected to the facets and that the function is non-stationary in these parts of the space. The problem cannot be resolved by simply shrinking the box bounds in the embedding. \citep{binois18} provides an excellent study of this issue and shows that with the REMBO strategy there is no good way to set box bounds in the embedding.

\paragraph{Linear projections do not preserve product kernels.} Although less visible than that produced by the projection to the facets, there is also distortion to interior points just from the linear projection $\mA$. The ARD kernels typically used in GP modeling are product kernels that decompose the covariance into that across each dimension. Inside the embedding, moving along a single dimension will move across all dimensions of the ambient space, at rates depending on the projection matrix. Thus a product kernel in the true subspace will not produce a product kernel in the embedding; this is shown mathematically in Proposition~\ref{prop:kernel}.

\paragraph{Linear embeddings can have a low probability of containing an optimum.} HeSBO avoids the challenges of REMBO related to box bounds: all interior points in the embedding map to interior points of $\mathcal{B}$, and there is no need for the $\normltwo$ projection and so the ability to model in the embedding is improved. However, for $d_e > 1$ the embedding is not guaranteed to contain an optimum with high probability, and in fact the probability of containing an optimum can be low. Consider the example of an axis-aligned true subspace: $f$ operates only on $d$ elements of $\vx$, denoted $\mathcal{I} = \{i_1, \ldots, i_d\}$. For $d=2$ and $d_e \geq 2$, there are three possible HeSBO embeddings: $x_{i_1}$ and $x_{i_2}$ map to different features in the embedding, $x_{i_1} = x_{i_2}$, or $x_{i_1} = -x_{i_2}$. These three embeddings are visualized in the supplement in \sect{appendix:hesbo}. In the first case the embedding successfully captures the entire true subspace and we can expect the optimization to be successful. However, in the other two cases the embedding is only able to reach the diagonals of the true subspace, which will not reach the optimal value, unless $f$ happens to have an optimum on the diagonal. Under a uniform prior on the location of optima, we can compute analytically the probability that the HeSBO embedding contains an optimum (see \sect{appendix:hesbo}).
For instance, with $d=6$, $d_e=12$ gives only a 22\% chance of recovering an optimum.
Relative to REMBO, HeSBO improves the ability to model and optimize in the embedding, but reduces the chance of the embedding containing an optimum. Empirically, this trade-off leads to HeSBO often having better HDBO performance than REMBO. Here we also wish to improve our ability to model and optimize in the embedding, which we will show can be done while maintaining a higher chance of the embedding containing an optimum, further improving HDBO performance.

\section{Learning and Optimizing in Linear Embeddings}
\label{sec:method}
We now describe how the embedding issues described in \sect{sec:remboissues} can be overcome.
Similarly to \citep{binois18}, we define the embedding via a matrix $\mB \in \mathbb{R}^{d_e \times D}$ that projects from the ambient space down to the embedding, and $f_B(\vy) = f(\mB^{\dagger}\vy)$ as the function evaluated on the embedding, where $\mB^{\dagger}$ denotes the matrix pseudo-inverse. The techniques we develop here are applicable to any linear embedding, not just random embeddings.

\paragraph{A Kernel for Learning in a Linear Embedding}
As discussed in \sect{sec:remboissues}, a product kernel over dimensions of the true subspace (\textit{e.g.}, ARD) does not translate to a product kernel over the embedding. This result gives the appropriate kernel structure.
\begin{proposition}\label{prop:kernel}
Suppose the function on the true subspace is drawn from a GP with an ARD RBF kernel: $f_d \sim \mathcal{GP}(m(\cdot), k_{\textrm{RBF}}(\cdot, \cdot))$. For any pair of points in the embedding $\vy$ and $\vy'$,
\begin{equation*}
    \textrm{Cov}[f_B(\vy), f_B(\vy')] = \sigma^2 \exp \left(-(\vy - \vy')^{\top} \bm{\Gamma} (\vy - \vy') \right)\,,
\end{equation*}
where $\sigma^2$ is the kernel variance of $f_d$, and $\bm{\Gamma} \in \mathbb{R}^{d_e \times d_e}$ is symmetric and positive definite.
\end{proposition}
\begin{proof}
To determine the covariance function in the embedding, we first project up to the ambient space and then project down to the true subspace: $f_B(\vy) = f(\mB^{\dagger}\vy) = f_d(\mT \mB^{\dagger}\vy)$. Then,
\begin{align*}
    \textrm{Cov}[f_B(\vy), f_B(\vy')] &= \textrm{Cov}[f_d(\mT \mB^{\dagger}\vy), f_d(\mT \mB^{\dagger}\vy')]\\
    &=\sigma^2 \exp \left(-(\mT \mB^{\dagger}\vy - \mT \mB^{\dagger}\vy')^{\top} \mD (\mT \mB^{\dagger}\vy - \mT \mB^{\dagger}\vy') \right),
\end{align*}
where $\mD = \text{diag}\left( \left[ \frac{1}{2 \ell_1^2}, \ldots, \frac{1}{2 \ell_d^2} \right] \right)$ are the inverse lengthscales. Let $\bm{\Gamma} = (\mT \mB^{\dagger})^{\top} \mD (\mT \mB^{\dagger})$. Because $\mD$ is positive definite, $\bm{\Gamma}$ is symmetric and positive definite.
\end{proof}
This kernel replaces the ARD Euclidean distance with a Mahalanobis distance, and so we refer to it as the Mahalanobis kernel. A similar result was found by \citep{garnett2014linear}, which showed that an RBF kernel in a linear subspace implies a Mahalanobis kernel in the high-dimensional, ambient space. Similar kernels have also been used for GP regression in other settings \citep{vivarelli1999discovering, snelson2016noise}.

Prop. \ref{prop:kernel} shows that the impact of the linear projection on the kernel can be correctly handled by fitting a $\frac{d_e(d_e+1)}{2}$-parameter distance metric rather than the typical $d_e$-parameter ARD metric. We handle uncertainty in $\bm{\Gamma}$ by posterior sampling from a Laplace approximation of its posterior; this is described in \sect{appendix:gamma} in the supplement.
The use of this kernel is vital for obtaining good model fits in the embedding, as shown in \fig{fig:ard_mahalanobis}. For \fig{fig:ard_mahalanobis}, a 6-d random linear embedding was generated for the Hartmann6 $D$=100 problem, and 100 training and 50 test points were randomly sampled that mapped to the interior of $\mathcal{B}$ (so, they have no distortion from clipping). The ARD RBF kernel entirely failed to learn the function on the embedding and simply predicted the mean; the Mahalanobis kernel made accurate out-of-sample predictions. Further details are given in \sect{appendix:gamma}, along with learning curves.

A similar argument as Prop. \ref{prop:kernel} shows that with linear embeddings, stationarity in the true subspace implies stationarity in the embedding; see \sect{appendix:stationary}. As discussed there, this result does not hold with clipping to box bounds, which effectively produces a nonlinear embedding.

\paragraph{Avoiding Nonlinear Projections}
The most significant distortions seen in \fig{fig:rembo_illustration} result from clipping projected points to $\mathcal{B}$. We can avoid this by constraining the optimization in the embedding to points that do not project up outside the bounds, that is, $\mB^{\dagger} \vy \in \mathcal{B}$. Let $\alpha(\vy)$ be the acquisition function evaluated in the embedding. We select the next point to evaluate by solving
\begin{equation}\label{eq:acq}
    \max_{\vy \in \mathbb{R}^{d_e}} \alpha(\vy) \quad\textrm{ subject to }\quad -\vone \leq \mB^{\dagger} \vy  \leq \vone\,.
\end{equation}
Box bounds on $\vy$ are not required. The constraints $-\vone \leq \mB^{\dagger} \vy  \leq \vone$ are all linear, so they form a polytope and can be handled with off-the-shelf optimization tools; we use Scipy's SLSQP. \sect{appendix:embedding} in the supplement provides visualizations of the embedding subject to these constraints. Within this space, the projection is entirely linear and can be effectively modeled with the Mahalanobis kernel.

\paragraph{The Probability the Embedding Contains an Optimum}

Restricting the embedding with the constraints in (\ref{eq:acq}) eliminates distortions from clipping to $\mathcal{B}$, but it also reduces the volume of the ambient space that can be reached from the embedding and thus reduces the probability that the embedding contains an optimum. 
To understand the performance of BO in the linear embedding, it is critical to understand this probability, which we denote $P_{\textrm{opt}}$. Recall that with clipping, the REMBO theoretical result does not hold when function evaluations are restricted to box bounds, and so even REMBO generally has $P_{\textrm{opt}} < 1$. We now describe how $P_{\textrm{opt}}$ can be estimated, and increased.

\begin{figure}[t]
	\centering
		\begin{minipage}[t]{0.48\textwidth}
	\centering
		\includegraphics{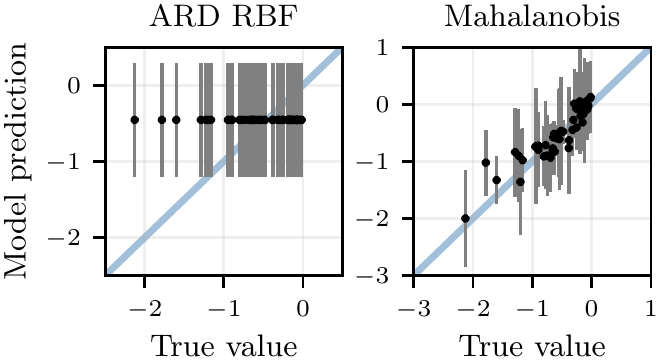}
		\caption{Predictions (mean, and in error bars two standard deviations, of the posterior predictive distribution) on a test set of 50 points from a 6-d embedding of the Hartmann6 $D$=100 problem, with models fit to 100 training points. The ARD RBF kernel is unable to learn in the embedding and predicts the mean. The Mahalanobis kernel makes accurate test-set predictions.}
		\label{fig:ard_mahalanobis}
		%\vspace{-10pt}
	\end{minipage}
	\hfill
	\begin{minipage}[t]{0.49\textwidth}
	\centering
		\includegraphics{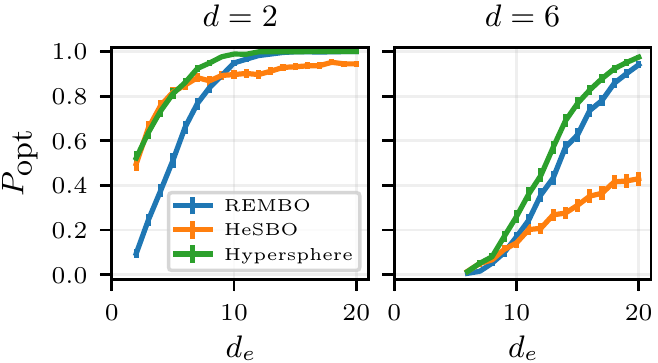}
		\caption{Probability the embedding contains an optimum ($P_{\textrm{opt}}$) when restricted to the constraints of (\ref{eq:acq}), under a uniform prior for the location of the optima and $D$=100, for three embedding strategies. Setting $d_e > d$ rapidly increases $P_{\textrm{opt}}$, and high probabilities can achieved with reasonable values of $d_e$. Hypersphere sampling produces the best embedding, particularly for $d$ small.}
		\label{fig:lp_solns}
		%\vspace{-10pt}
	\end{minipage}
\end{figure}

$P_{\textrm{opt}}$ depends on where optima are in the ambient space---for instance, an optimum at $\vzero$ is always contained in the embedding. Suppose the true subspace has an optimum at $\vz^*$. Then, $\mathcal{O}(\mT, \vz^*) = \{\vx : \mT \vx = \vz^*\}$ is the set of optima in the ambient space. We wish to determine if these can be reached from the embedding. The points $\vx$ that can be reached from the embedding are those for which there exists a $\vy$ in the embedding that projects up to $\vx$, that is, $\mB^{\dagger} \vy = \vx$. Since the embedding is produced from $\mB \vx$,
the points that can be reached from the embedding are $\mathcal{E}(\mB) = \{\vx : \mB^{\dagger} \mB \vx = \vx \}$. The embedding contains an optimum if and only if the intersection $\mathcal{O}(\mT, \vz^*) \cap \mathcal{E}(\mB) \cap \mathcal{B}$ is non-empty. Given a prior for the locations of optima (that is, over $\mT$ and $\vz^*$),
\begin{equation}\label{eq:p_contains_optimizer}
    P_{\textrm{opt}} = \mathbb{E}_{\mB, \mT, \vz^*} \left[ \1_{\mathcal{O}(\mT, \vz^*) \cap \mathcal{E}(\mB) \cap \mathcal{B} \neq \varnothing}  \right]\,.
\end{equation}
Importantly, $\mathcal{O}(\mT, \vz^*)$, $\mathcal{E}(\mB)$, and $\mathcal{B}$ are all polyhedra, so their intersection can be tested by solving a linear program (see \sect{appendix:lp} in the supplement). The expectation can be estimated with Monte Carlo sampling from the prior over $\mT$ and $\vz^*$, and from the chosen generating distribution of $\mB$.

For our analysis here, we give $\mT$ a uniform prior over axis-aligned subspaces as described in \sect{sec:rembo}, and we give $\vz^*$ a uniform prior in that subspace. Under these uniform priors, we can evaluate (\ref{eq:p_contains_optimizer}) to compute $P_{\textrm{opt}}$ as a function of $\mB$, $D$, $d$, and $d_e$.
\fig{fig:lp_solns} shows these probabilities for $D$=100 as a function of $d$ and $d_e$, with three strategies for generating the projection matrix: the REMBO strategy of $\mathcal{N}(0, 1)$, the HeSBO projection matrix, and the unit hypersphere sampling described in \sect{sec:background}. Increasing $d_e$ above $d$ rapidly improves the probability of containing an optimum. For $d=6$, with $d_e=6$ $P_{\textrm{opt}}$ is nearly 0, while increasing $d_e$ to 12 is sufficient to raise $P_{\textrm{opt}}$ to 0.5 and with $d_e=20$ it is nearly 1. Across all values of $d$ and $d_e$, hypersphere sampling produces the embedding with the best chance of containing an optimum. 
\sect{appendix:lp} shows $P_{\textrm{opt}}$ for more values of $D$ and $d$. Hypersphere sampling and selecting $d_e > d$ are important techniques for maintaining a high $P_{\textrm{opt}}$.

\paragraph{A New Method for BO with Linear Embeddings}

We combine the results and insight gained above into a new method for HDBO, called adaptive linear embedding BO (ALEBO), since the kernel metric and embedding bounds adapt with $\mB$. It is summarized in Algorithm \ref{algo:method}.

\begin{algorithm2e}[tbh]
\KwData{$D$, $d_e$, $n_{\textrm{init}}$, $n_{\textrm{BO}}$.}
\KwResult{Approximate optimizer $\vx^*$.}
Generate a random projection matrix $\mB$ by sampling $D$ points from the hypersphere $\mathbb{S}^{d_e-1}$.

Generate $n_{\textrm{init}}$ random points $\vy^i$ in the embedding using rejection sampling to satisfy the constraints of (\ref{eq:acq}).

Let $\mathcal{D} = \{(\vy^i, f(\mB^{\dagger} \vy^i) \}_{i=1}^{n_{\textrm{init}}}$ be the initial data.

\For{$j = 1, \ldots, n_{\textrm{BO}}$}{
Fit a GP to $\mathcal{D}$ with the Mahalanobis kernel, using posterior sampling (\sect{appendix:gamma}).

Use the GP to find $\vy^j$ that maximizes the acquisition function according to (\ref{eq:acq}).

Update $\mathcal{D}$ with $(\vy^j, f(\mB^{\dagger} \vy^j))$.
}
\Return{$\mB^{\dagger} \vy^*$, for the best point $\vy^*$.}
\caption{ALEBO for linear embedding BO.}
\label{algo:method}
\end{algorithm2e}

\section{Benchmark Experiments}
\label{sec:experiments}
\begin{figure}[t]
    \centering
    \includegraphics{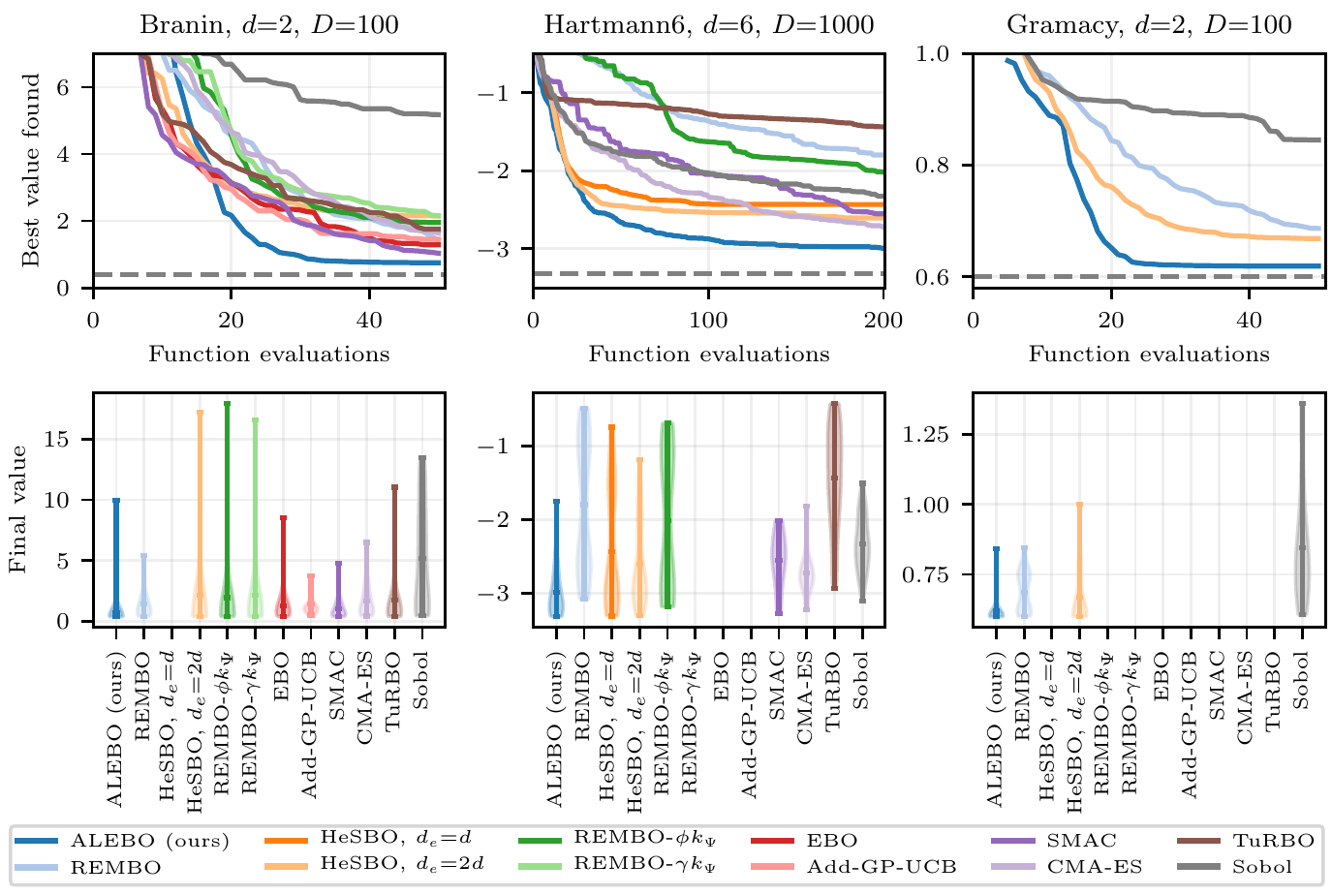}
    %\vspace{-5pt}
    \caption{Optimization (minimization) performance on three HDBO benchmarks. (\textit{Top}) Best value by each iteration, averaged over repeated runs. (\textit{Bottom}) Distribution of the final best value. For all three tasks, ALEBO achieved the best average performance, with low variance.}
    \label{fig:synthetic_results}
    \vspace{-5pt}
\end{figure}

Past work has shown REMBO can perform poorly even on problems that have a true, linear subspace, despite this being the setting that REMBO should be best-suited for. A major source of this poor performance are the modeling issues described in \sect{sec:rembo}, which we demonstrate by showing that with the new developments in \sect{sec:method}, ALEBO can achieve state-of-the-art performance on this class of problems (those with a true linear subspace). We compare performance to a broad selection of existing methods: REMBO; REMBO variants $\phi k_\Psi$, $\gamma k_\Psi$, and HeSBO; additive kernel methods Add-GP-UCB \citep{kandasamy2015high} and Ensemble BO (EBO) \citep{wang2018batched}; SMAC; CMA-ES, an evolutionary strategy \citep{hansen03}; TuRBO, trust region BO \citep{eriksson19}; and quasirandom search (Sobol). \sect{appendix:experiments} additionally compares to LineBO \citep{kirschner19}. For ALEBO we took $d_e=2d$ for these experiments. In their evaluation of HeSBO, \citep{nayebi19hesbo} used $d_e=2d$ for $d=2$ but $d_e=d$ on the Hartmann6 problem; we evaluate both choices.

\fig{fig:synthetic_results} shows optimization performance for three HDBO tasks, averaged over 50 runs: Branin extended to $D$=100, Hartmann6 extended to $D$=1000, and Gramacy extended to $D$=100. The Gramacy problem \citep{gramacy16} includes two black-box constraints, which are naturally handled by linear embedding methods (see \sect{appendix:constrained}; an advantage of linear embedding HDBO is that it is agnostic to the acquisition function). For the $D$=1000 problem, REMBO-$\gamma k_\Psi$, EBO, and Add-GP-UCB did not finish a single run after 24 hours and so were terminated. These methods, along with TuRBO, SMAC, and CMA-ES, also do not support blackbox constraints and so were not used for Gramacy. 
\sect{appendix:experiments} provides additional details, additional experimental results (plots of log regret and error bars), two additional benchmark problems (including a non-axis-aligned problem), and an extended discussion of the results.

Consistent with past studies, REMBO performance was variable, and in the $D$=1000 problem it performed worse than random. With the adjustments in \sect{sec:method}, ALEBO significantly improved HDBO performance relative to other linear embedding methods, and achieved the best average optimization performance overall. ALEBO also had low variance in the final best-value, which is important in real applications where one can typically only run one optimization run. These results show that with the adjustments in ALEBO, linear embedding BO is the best-performing method on linear subspace benchmark problems, as it ought to be. \sect{appendix:experiments} gives a sensitivity analysis with respect to $D$ and $d_e$ (robust for $d_e>d$), and an ablation study of the different components of ALEBO.

\section{Real-World Problems}\label{sec:realworld}
\paragraph{Constrained Neural Architecture Search}
\begin{wrapfigure}{r}{0.54\linewidth}
    \centering
    \vspace{-12pt}
    \includegraphics{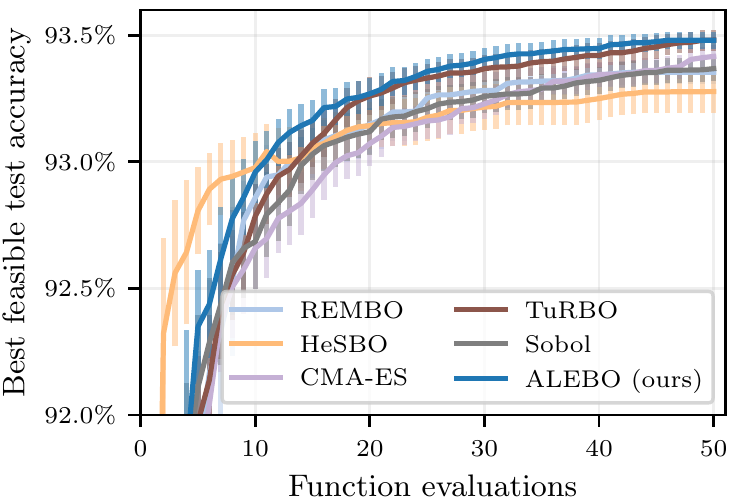}
%\vspace{-10pt}
\caption{Best-feasible CIFAR10 test-set accuracy by each iteration for the $D=36$ constrained NAS problem, showing mean and two standard errors over 100 repeated runs. ALEBO was a best-performing method, and the only embedding method that outperformed random search.
}
\vspace{-10pt}
\label{fig:nasbench}
\end{wrapfigure}
We evaluated ALEBO performance on constrained neural architecture search (NAS) for convolutional neural networks using models from NAS-Bench-101 \citep{nasbench}. The NAS problem was to design a cell topology defined by a DAG with 7 nodes and up to 9 edges, which includes designs like ResNet \cite{resnet} and Inception \cite{inception}. We created a $D=36$ parameterization, producing a HDBO problem. The objective was to maximize CIFAR-10 test-set accuracy, subject to a constraint that training time was less than 30 mins; see Sec. \ref{appendix:nas} for full details. Sample efficiency is critical for NAS due to long training times on GPUs or TPUs. Fig. \ref{fig:nasbench} shows optimization performance on this problem. ALEBO significantly improved accuracy relative to the other linear embedding approaches, which did not improve over Sobol, and was a best-performing method.
\paragraph{Policy Search for Robot Locomotion}
We applied ALEBO to the problem of learning walking controllers for a simulated hexapod robot. Sample efficiency is crucial in robotics as collecting data on real robots is time consuming and can cause wear-and-tear on the robot. We optimized the walking gait of the ``Daisy'' robot \citep{RoboticsDaisy}, which has 6 legs with 3 motors in each leg, and was simulated in PyBullet~\citep{pybullet}.
The goal was to learn policy parameters that enable the robot to walk to a target location while avoiding high joint velocities and height deviations; details are given in \sect{sec:appendix:daisy}.
We use a Central Pattern Generator (CPG)~\citep{crespi2008online} with $D=72$ to control the robot. The $72$-dimensional controller assumes each joint is independent of the others; a lower-dimensional embedding can be constructed by coupling multiple joints. For example, the tripod gait in hexapods assumes three sets of legs synced and out of phase with the remaining three legs, which produces an $11$-dimensional parameterization.
The existence of such low-dimensional parameterizations motivates the use of embedding methods for this problem, though there is no known \textit{linear} low-dimensional representation.

\fig{fig:expdaisy} shows optimization performance on this task. ALEBO performed the best of the linear embedding methods, and also outperformed EBO, SMAC, and Sobol. REMBO performed poorly on this problem, only slightly better than random. The ALEBO results show that REMBO did not do poorly because linear embedding methods cannot learn on this problem, rather it was because of the issues described and corrected in this paper.
CMA-ES outperformed all of the HDBO methods. CMA-ES is model free, suggesting that the underlying models in the HDBO methods are not well suited for this particular, real problem. This is likely due to discontinuities in the function: in some parts of the space a small perturbation can cause the robot to fall and significantly reduce the reward, while in other parts the reward will be much smoother. Expert tuning of the tripod gait can achieve reward values above 40, much better than any gait found in the optimizations here. ALEBO enables linear embedding methods to reach their full potential, but these results show that there is still much room for additional work in HDBO.

\begin{figure}[t]
	\centering
	\includegraphics{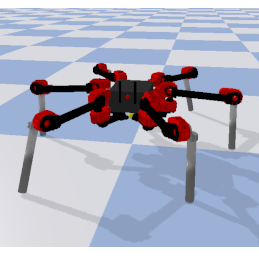}
	\hspace{.12in}
	\includegraphics{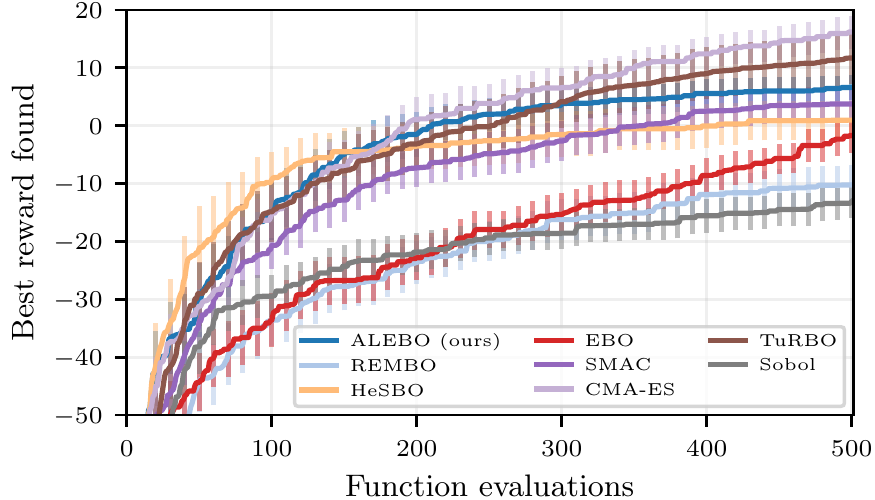}
	\caption{(\textit{Left}) The simulated hexapod robot Daisy. (\textit{Right}) Optimization (maximization) performance on the $D=72$ locomotion task, showing mean and two standard errors over 100 runs. ALEBO significantly improved over REMBO, but all HDBO methods were outperformed by CMA-ES.}
	\label{fig:expdaisy}
\end{figure}

\section{Conclusion}
\label{sec:discussion}
Our work highlights the importance of two basic requirements for an embedding to be useful for optimization that are often not examined critically by the literature: 1) the function must be well-modeled on the embedding; and 2) the embedding should contain an optimum. To the first point, we showed how polytope constraints on the embedding eliminate boundary distortions, and we derived a Mahalanobis kernel appropriate for GP modeling in a linear embedding.
To the second, we developed an approach for computing the probability that the embedding contains an optimum, which we used to construct embeddings with a high chance of containing an optimum, via hypersphere sampling and selecting $d_e > d$. With ALEBO we verified empirically that addressing these issues resolved the poor REMBO performance on benchmark tasks and real-world problems.

These same considerations are important for any embedding. When constructing a VAE for BO it will be equally important to ensure the function remains well-modeled in the embedding, to handle box bounds in an appropriate way, and to ensure the embedding has a high chance of containing an optimum.
%With linear embeddings we were able to derive analytical quantities for answering these questions---more work in this area is needed for nonlinear embeddings.
End-to-end learning of the embedding and the GP can produce an embedding amenable to modeling, though potentially requiring a large number of evaluations to learn the embedding. Adapting nonlinear embeddings learned from auxiliary data for HDBO is an important area of future work. Clipping to box bounds will hurt modelability in a nonlinear embedding in the same way as in a linear embedding. Here we incorporated linear constraints into the acquisition function optimization; for a VAE these constraints will be nonlinear, but their gradients can be backpropped and so constrained optimization can be done in a similar way.

The experiment results show that linear embedding HDBO, and ALEBO in particular, can be a valuable tool for high-dimensional optimization. On real-world problems, we found that local-search methods (CMA-ES and TuRBO) can be highly competitive, or even best-performing. These particular problems have discontinuities that make global modeling difficult and favor local search, however there are other settings where embedding HDBO will be the best choice. The appeal of using embeddings for HDBO is that all of the BO techniques developed over the past two decades can be directly applied to high-dimensional problems. Settings where BO is not matched by local search include cost-aware \citep{snoek2012practical}, multi-task \citep{swersky13mtbo}, and multi-fidelity \citep{wu20mf} optimization. These methods can be directly adapted to HDBO by applying them inside a random linear embedding, and the techniques described in this paper will ensure the best possible performance.

%===============================================================================

\clearpage % this section does not count towards the eight pages of content that are allowed.
\section*{Broader Impact}

Bayesian optimization is a powerful optimization technique used in a wide range of industries and applications, such as robotics~\citep{Lizotte2007Automatic,Calandra2015a,Rai2018Bayesian}, internet tech companies \citep{vizier, letham2019noisyei}, designing novel molecules for pharmaceutics \citep{gomez2018chemical}, material design for increasing efficiency of solar cells \citep{zhang20}, and aerospace engineering \citep{lam18}.
All of these settings have high-dimensional optimization problems, and advances in BO will reflect on improved capabilities on these fields as well. 
We have fully open-sourced our code for ALEBO to be available for researchers and practitioners in these fields, and many others. 
The ability to optimize a larger number of parameters than has previously been possible will bring further improvements to and further accelerate work in these areas.

% Authors are required to include a statement of the broader impact of their work, including its ethical aspects and future societal consequences. 
% Authors should discuss both positive and negative outcomes, if any. For instance, authors should discuss a) 
% who may benefit from this research, b) who may be put at disadvantage from this research, c) what are the consequences of failure of the system, and d) whether the task/method leverages
% biases in the data. If authors believe this is not applicable to them, authors can simply state this.

%===============================================================================

\section*{Acknowledgements}
R.C. thanks Marc Deisenroth and Frank Hutter for insightful discussions, and Victor-Philipp Negoescu, Mark Prediger, Florian Schnell for preliminary experiments back in 2013.

\bibliographystyle{abbrvnat}
\bibliography{refs}

\newpage
\onecolumn
\begin{center}
\textbf{\large Supplemental Materials: Re-Examining Linear Embeddings for High-Dimensional Bayesian Optimization}
\end{center}
\setcounter{section}{0}
\setcounter{proposition}{0}
\setcounter{equation}{0}
\setcounter{figure}{0}
\setcounter{table}{0}
\setcounter{page}{1}
\renewcommand{\thesection}{S\arabic{section}}
\renewcommand{\theproposition}{S\arabic{proposition}}
\renewcommand{\theequation}{S\arabic{equation}}
\renewcommand{\thefigure}{S\arabic{figure}}
\renewcommand{\thetable}{S\arabic{table}}
\renewcommand{\thealgocf}{S\arabic{algocf}} 
This supplemental material contains a number of additional results and analyses to support the main text.

\section{HeSBO Embeddings}\label{appendix:hesbo}

We consider HeSBO embeddings in the case of a random axis-aligned true subspace, and a uniform prior on the location of the optimum within that subspace. As explained in \sect{sec:rembo}, with $d=2$ and this prior, regardless of $d_e$ or $D$ there are three possible embeddings: (1) each of the active parameters are captured by a parameter in the embedding; (2) the embedding is constrained to the diagonal $x_{i_1} = x_{i_2}$; or (3) the embedding is constrained to the diagonal $x_{i_1} = -x_{i_2}$. Fig. \ref{fig:hesbo_embeddings} shows these three embeddings for the Branin problem from the top row of Fig. \ref{fig:rembo_illustration}.

\begin{figure}[!b]
    \centering
    \includegraphics[width=5.5in]{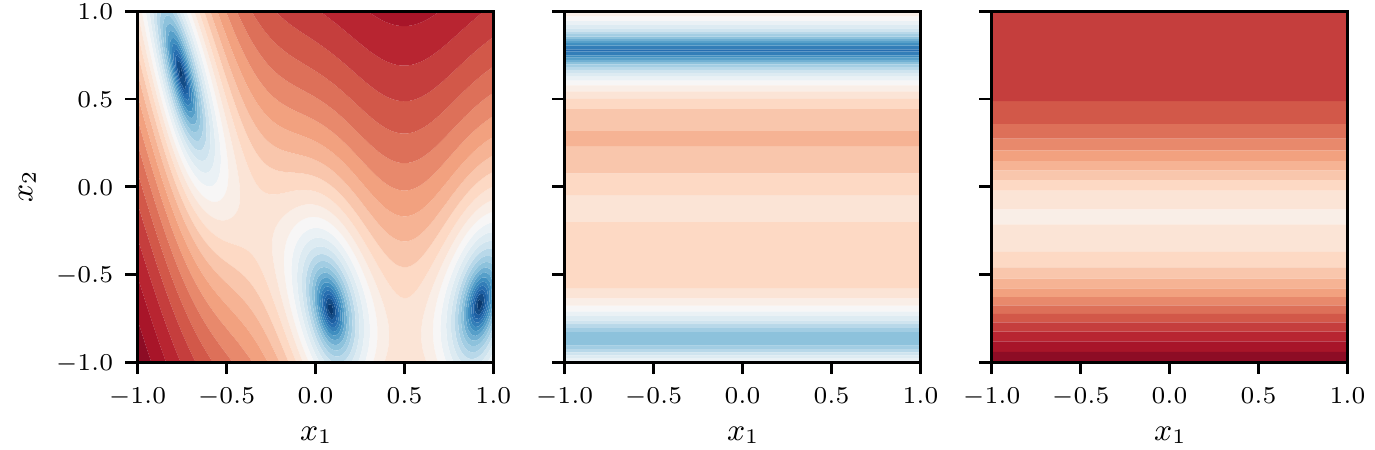}
    \caption{Three possible HeSBO embeddings of the $d=2$ Branin function. \textit{(Left}) The first embedding fully captures the function, and thus captures all three optima. \textit{(Middle)} The second is restricted to the subspace $x_1 = -x_2$. This subspace does not contain an optimum, but comes fairly close. \textit{(Right)} The third embedding is restricted to the subspace $x_1=x_2$ and does not come close to any optimum.}
    \label{fig:hesbo_embeddings}
\end{figure}

Within the first embedding, the optimal value of 0.398 can be reached. Within the second, the best value is 0.925 and within the third it is 17.18. Under a uniform prior on the location of the optimum within a random axis-aligned true subspace, it is easy to compute the probability that the HeSBO embedding contains an optimum:
\begin{equation}\label{eq:hesbo_prob}
    P_{\textrm{opt}}(d_e) = \frac{d_e!}{(d_e-d)! d_e^d}.
\end{equation}
For $d=2$, this is exactly the probability of the first embedding shown in Fig. \ref{fig:hesbo_embeddings}. This probability does not depend on $D$, but increases with $d_e$, and is the probability shown in Fig. \ref{fig:lp_solns}.

\section{The Mahalanobis Kernel}\label{appendix:gamma}
When fitting the Mahalanobis kernel derived in Proposition \ref{prop:kernel}, we use an approximate Bayesian treatment of $\bm{\Gamma}$ to improve model performance while still maintaining tractability. We propagate uncertainty in $\bm{\Gamma}$ into the GP posterior by first constructing a posterior for $\bm{\Gamma}$ using a Laplace approximation with a diagonal Hessian, and then drawing $m$ samples from that posterior. The marginal posterior for $f(\vy)$ can then be approximated as:
\begin{equation*}
    p(f(\vy)) \approx \frac{1}{m} \sum_{i=1}^m p(f(\vy) | \bm{\Gamma}^i).
\end{equation*}
Because of the GP prior, each conditional posterior $p(f(\vy) | \bm{\Gamma}^i)$ is a normal distribution with known mean $\mu_i$ and variance $\sigma_i^2$. Thus the posterior $p(f(\vy))$ is a mixture of Gaussians, which we can approximate using moment matching:
\begin{equation*}
    p(f(\vy)) \approx \mathcal{N} \left( \frac{1}{m} \sum_{i=1}^m \mu_i ,  \frac{1}{m} \sum_{i=1}^m \sigma^2_i   + \textrm{Var}_i[\mu_i] \right).
\end{equation*}
We do this to maintain a Gaussian posterior, under which acquisition functions like EI have analytic form and can easily be optimized, even subject to constraints as in (\ref{eq:acq}).

As described in \sect{sec:method}, we show the importance of the Mahalanobis kernel using models fit to data from the Hartmann6 $D$=100 function. We generated a projection matrix $\mB$ using hypersphere sampling to define a 6-d linear embedding. We then generated a training set (100 points) and a test set (50 points) within that embedding---that is, within the polytope given by (\ref{eq:acq})---using rejection sampling. We fit three GP models with different kernels to the training set, and then evaluated each on the test set: a typical ARD RBF kernel in 6 dimensions, the Mahalanobis kernel using a point estimate for $\bm{\Gamma}$, and the Mahalanobis kernel with posterior marginalization for $\bm{\Gamma}$ as described above.

\begin{figure}
\centering
%\vspace{-10pt}
\includegraphics{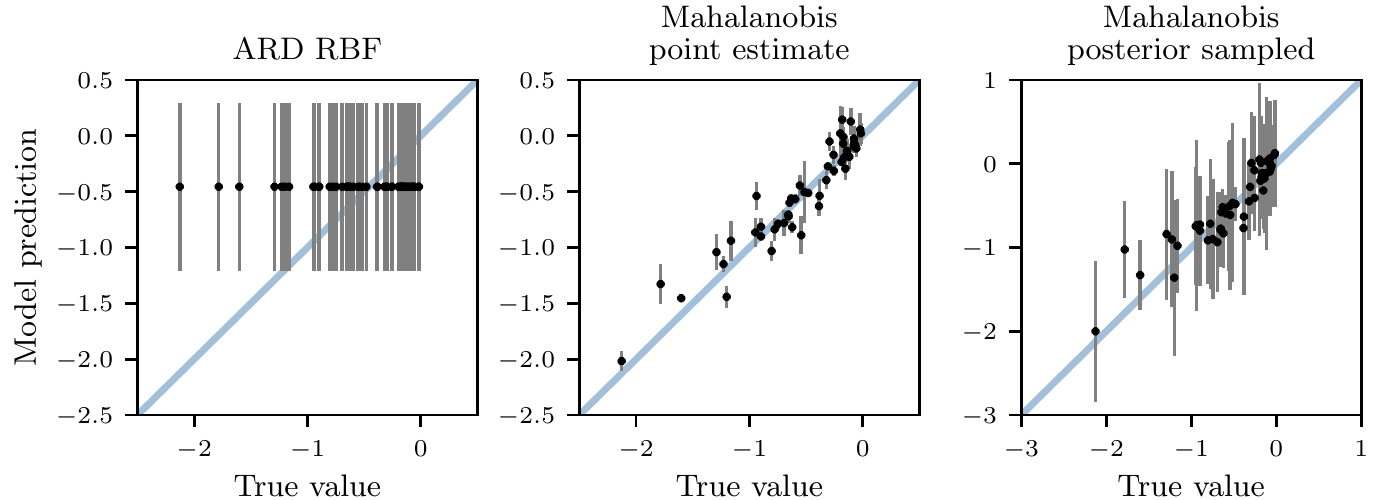}
\caption{Test-set model predictions for three GP kernels on the same train/test data generated by evaluating the Hartmann6 $D$=100 function on a fixed linear embedding. A typical ARD kernel fails to learn and predicts the mean. The Mahalanobis kernel predicts well, and posterior sampling is important for getting reasonable predictive variance.}
\label{fig:model_predictions}
\end{figure}

\fig{fig:model_predictions} compares model predictions for each of these models with the actual test-set outcomes; results here are the same as in \fig{fig:ard_mahalanobis} with the addition of the Mahalanobis point estimate kernel. With an ARD RBF kernel, the GP predicts the function mean everywhere, which is typical behavior of a GP that has failed to learn the function. With the same training data, the Mahalanobis kernel is able to make accurate predictions on the test set. Using a point estimate for $\bm{\Gamma}$ significantly underestimates the predictive variance, which is rectified by using posterior sampling as described above. In BO exploration is driven by model uncertainty, so well-calibrated uncertainty intervals are especially important.

\begin{figure}
\centering
%\vspace{-10pt}
\includegraphics{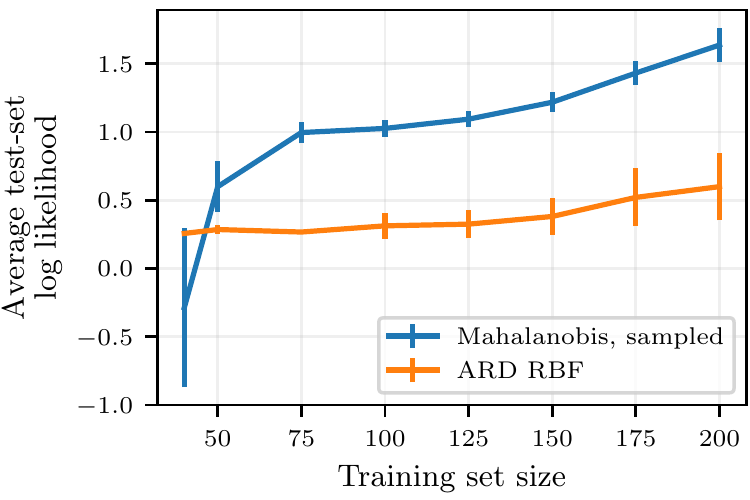}
\caption{Average test-set log likelihood as a function of training set size, for training sets randomly sampled from a fixed linear embedding. Log marginal probabilities were averaged over a fixed test set of 1000 random points. For each training set size, 20 random training sets were drawn of that size and the figure shows the average result over those draws (with error bars for two standard errors). The ARD RBF kernel continues to predict the mean as the training set size is increased, while the Mahalanobis kernel is able to learn as the training set is expanded.}
\label{fig:log_likelihood}
\end{figure}

\fig{fig:log_likelihood} evaluates the predictive log marginal probabilities for the ARD RBF kernel and the Mahalanobis kernel with posterior sampling across a wide range of training sets with different sizes (without posterior sampling, \fig{fig:model_predictions} shows that the Mahalanobis point estimate significantly under covers and so has very poor predictive log marginal probabilities). We used the same linear embedding and Hartmann6 $D$=100 function used in \fig{fig:model_predictions} to sample 1000 test points which were held fixed. For each of 8 training set sizes ranging from 40 to 200, we randomly sampled 20 training sets from the embedding. For each training set, we fit the two GPs, made predictions on the 1000 test points, and then computed the average marginal log probability of the true values. \fig{fig:log_likelihood} shows that as the training set size increased from 40 to 200, the ARD RBF kernel could only improve slightly on predicting the mean, as it did in \fig{fig:model_predictions}; even 200 points in the 6-d embedding were not sufficient to significantly improve the model. For small training set sizes, the Mahalanobis kernel (with sampling) had high variance in log likelihood, as it has the potential to overfit and thus under cover. But for training set sizes of 50 and greater it had better predictive log likelihood than the ARD RBF kernel, and continued to learn as the training set size was increased. For small datasets, the Mahalanobis kernel can overfit and thus have poor predictive likelihood, but for the purposes of BO, overfitting can be better than not fitting at all (predicting the mean), even when predicting the mean has better predictive log likelihood. This can be seen in the optimization results (Figs. \ref{fig:synthetic_results} and \ref{fig:log_regrets}) where ALEBO showed strong performance even with less than 50 iterations.

\section{Stationarity in the Embedding}\label{appendix:stationary}
A stationary kernel is one that depends only on $\vx - \vx'$, not on the individual values of $\vx$ and $\vx'$, and is thus invariant to translation \citep{rasmussen06}. The following result shows that with linear embeddings, stationarity in the true function implies stationarity in the embedding.
\begin{proposition}\label{prop:stationary}
Suppose the function on the true subspace is drawn from a GP with a stationary kernel: $f_d \sim \mathcal{GP}(m(\cdot), k(\cdot, \cdot))$ where $k(\vz, \vz') = \kappa(\vz - \vz')$. Let $g(\vy) = \mT \mB^{\dagger}\vy$. For any pair of points in the embedding $\vy$ and $\vy'$,
\begin{equation*}
    \textrm{Cov}[f_B(\vy), f_B(\vy')] = \tilde{\kappa}(\vy - \vy')\,,
\end{equation*}
where $\tilde{\kappa} = \kappa \circ g$. The implied kernel on the embedding is thus stationary.
\end{proposition}
\begin{proof}
The argument follows that of Prop. \ref{prop:kernel}. As shown there, $f_B(\vy) = f_d(\mT \mB^{\dagger}\vy)$. Then,
\begin{align*}
    \textrm{Cov}[f_B(\vy), f_B(\vy')] &= \textrm{Cov}[f_d(\mT \mB^{\dagger}\vy), f_d(\mT \mB^{\dagger}\vy')]\\
    &=\kappa(\mT \mB^{\dagger}\vy - \mT \mB^{\dagger}\vy')\\
    %&=\kappa(\mT \mB^{\dagger}(\vy - \vy'))\\
    &=\tilde{\kappa}(\vy - \vy').
\end{align*}
\end{proof}
Consider now clipping to box bounds in the ambient space with the $\normltwo$ projection $p_{\mathcal{B}}$. Then, $f_B(\vy) = f_d(\mT p_{\mathcal{B}}(\mB^{\dagger}\vy))$, and the implied kernel in the embedding is
\begin{equation*}
    \textrm{Cov}[f_B(\vy), f_B(\vy')] = \kappa(\mT (p_{\mathcal{B}}(\mB^{\dagger}\vy) - p_{\mathcal{B}}(\mB^{\dagger}\vy'))).
\end{equation*}
This is clearly non-stationary, because $p_{\mathcal{B}}$ is not translation invariant.

\section{Polytope Bounds on the Embedding}\label{appendix:embedding}

Rather than using projections to the box bounds $\mathcal{B}$, we specify polytope constraints in (\ref{eq:acq}). Fig. \ref{fig:new_embedding} illustrates the embedding with these constraints for the same Branin $D=100$ problem from the top row of Fig. \ref{fig:rembo_illustration}. The embedding in the left figure was created with the REMBO strategy of sampling each entry from $\mathcal{N}(0, 1)$. For the embedding in the right figure, that same projection matrix had each column normalized. This converts the projection matrix to be a sample from the unit circle, as described in Sec. \ref{sec:rembo}.

\begin{figure}
    \centering
    \includegraphics{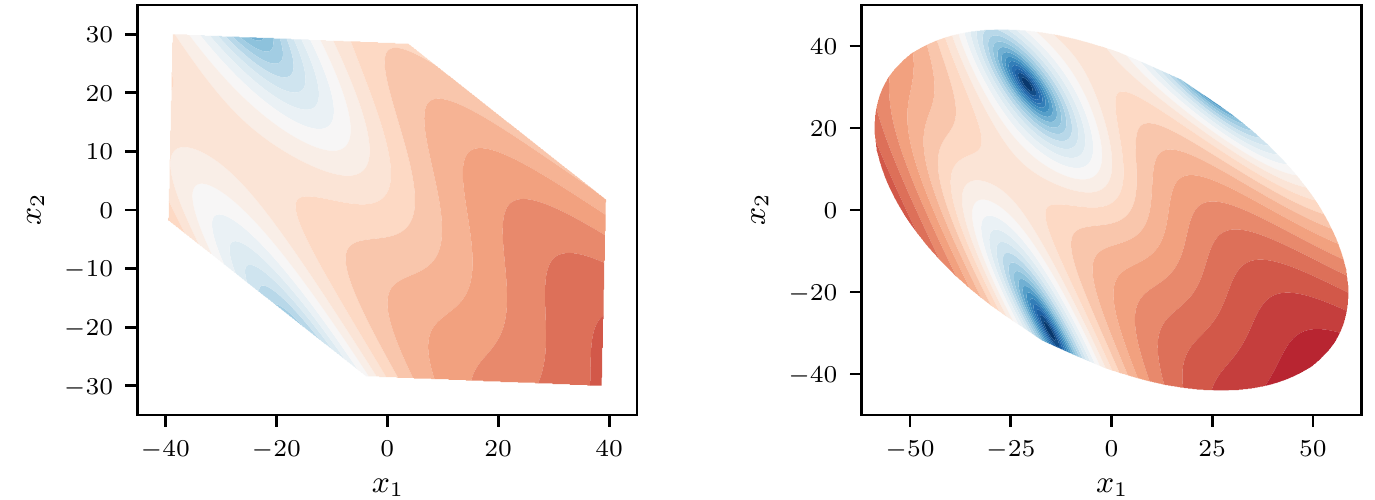}
    \caption{\textit{(Left)} An embedding from a $\mathcal{N}(0, 1)$ projection matrix on the same Branin $D=100$ problem from Fig. \ref{fig:rembo_illustration} subject to constraints of (\ref{eq:acq}). \textit{(Right)} The embedding from the same projection matrix after normalizing the columns to produce unit circle samples. Sampling from the unit circle increases the probability that an optimum will fall within the embedding, and polytope bounds avoid nonlinear distortions.}
    \label{fig:new_embedding}
\end{figure}

The $\mathcal{N}(0, 1)$ embedding does not contain any optima within the polytope bounds. Converting that projection matrix to a hypersphere sample rounds out the vertices of the polytope and expands the space to capture two of the optima. Consistent with Fig. \ref{fig:lp_solns}, we see that hypersphere sampling significantly improves the chances of the embedding containing an optimum. \fig{fig:new_embedding} also shows that with the polytope bounds, we avoid the nonlinear distortions seen with REMBO in \fig{fig:rembo_illustration}.

Note that adding linear constraints to a non-convex optimization problem (acquisition function optimization) does not change the complexity of that problem.

\section{Evaluating the Probability the Embedding Contains an Optimum}\label{appendix:lp}
As in other parts of the paper, we consider a uniform prior on the location of the optimum within a random axis-aligned subspace. A random true projection matrix $\mT$ is sampled by selecting $d$ columns at random and setting each to one of the $d$-dimensional unit vectors. $\vz^*$ is then sampled uniformly at random from $[-1, 1]^d$. $\mB$ is sampled according to the desired strategy, which in our experiments was REMBO ($\mathcal{N}(0, 1)$ entries), HeSBO, or hypersphere. Given these three quantities, we can evaluate whether or not the embedding contains an optimum that satisfies the constraints of (\ref{eq:acq}) by solving the following linear program:
\begin{align*}
    \textrm{maximize  } &\vzero^{\top} \vx \\
    \textrm{subject to  } & \mT \vx = \vz^*,\\
    & (\mB^{\dagger} \mB - \mI) \vx  = \vzero,\\
    & \vx \geq -\vone,\\
    & \vx \leq \vone.
\end{align*}

If this problem is feasible, then the embedding produced by $\mB$ contains an optimum; if it is infeasible, then it does not. Solving this over many draws of $\mT$, $\vz^*$, and $\mB$ produces an estimate of $P_{\textrm{opt}}$ under that prior for the location of optima. Here we used a uniform prior, but this linear program can be taken to compute $P_{\textrm{opt}}$ under any prior.

\begin{figure}
    \centering
    \includegraphics{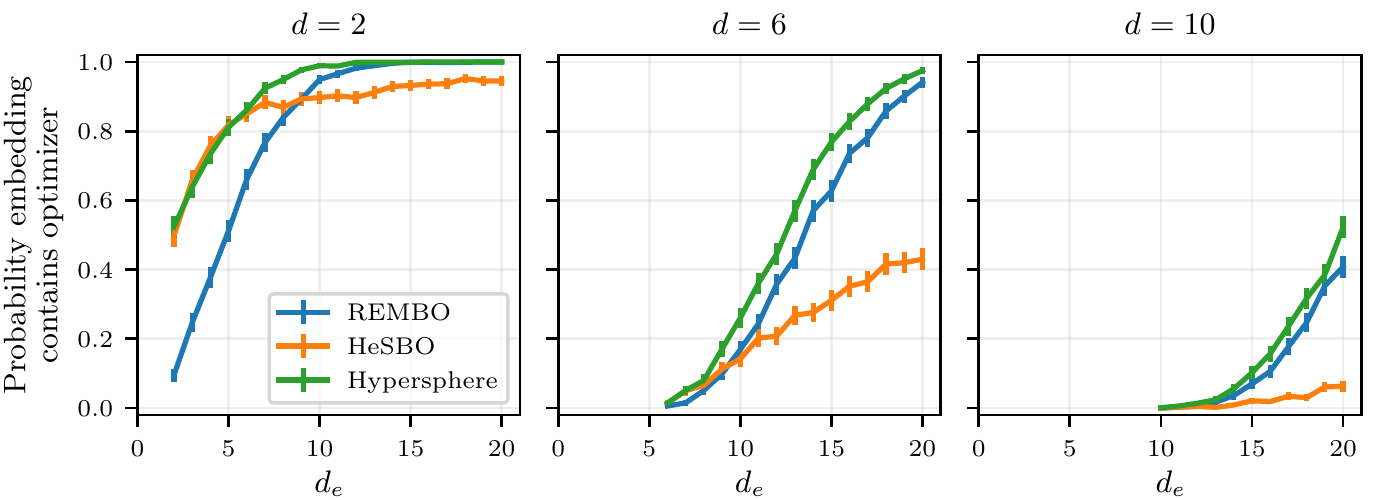}
    \caption{$P_{\textrm{opt}}$ as estimated in \fig{fig:lp_solns}, extended with the results for $d=10$. Setting $d_e>d$ significantly improves the probability of the embedding containing an optimum.}
    \label{fig:lp_solns_ext}
\end{figure}

\fig{fig:lp_solns_ext} shows $P_\textrm{opt}$ for the three embedding strategies as a function of $d$ and $d_e$, for $D$ fixed at 100. The results shown for $d=2$ and $d=6$ are those given in the main text in \fig{fig:lp_solns}. Fig. \ref{fig:lp_solns_D} shows $P_{\textrm{opt}}$ for a wide range of values of $d$ and $D$, for hypersphere sampling. Across this wide range we see that for many values of $d$ we can achieve high values of $P_{\textrm{opt}}$ with reasonable values of $d_e$, even for relatively high values of $D$.

\begin{figure}
    \centering
    \includegraphics{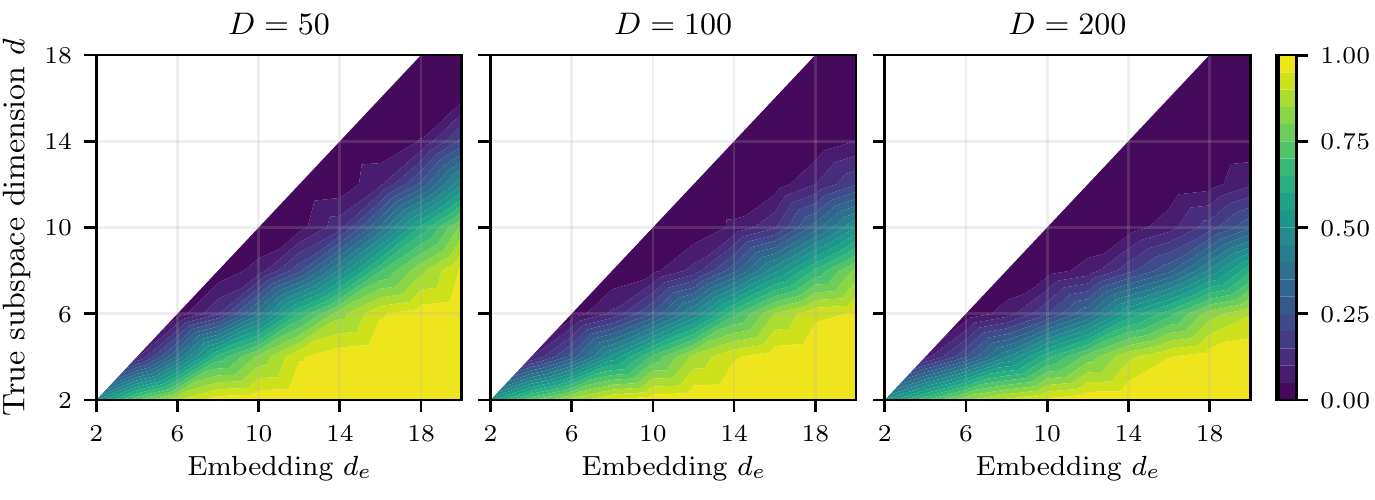}
    \caption{$P_{\textrm{opt}}$ for hypersphere sampling, as estimated in Fig. \ref{fig:lp_solns} but here for a wider range of values of $d$ and $D$. Contour color indicates $P_{\textrm{opt}}$. Doubling $D$ decreases $P_{\textrm{opt}}$ for $d$ and $d_e$ fixed, however even at $D=200$, high values of $P_{\textrm{opt}}$ with reasonable values of $d_e$ can be had for many values of $d$.}
    \label{fig:lp_solns_D}
\end{figure}

\section{Selecting the Embedding Dimension}\label{appendix:d_e}
Linear embedding HDBO requires selecting a dimensionality for the embedding. The results in Sections \ref{sec:method} and \ref{appendix:lp} show clearly that choosing an embedding dimensionality higher than that of the true subspace is vital for obtaining a high probability of the the embedding containing an optimum. In principal, if one knew the true subspace dimension $d$, the results of Sec. \ref{appendix:lp} could be used to calculate $P_{\textrm{opt}}$ as a function of $d_e$, and then $d_e$ could be chosen to reach a desired value of $P_{\textrm{opt}}$. In practice, however we will not typically know what the true subspace dimension is, or even be certain of the existence of a true, linear subspace.

In real BO problems, which have expensive function evaluations, there is always a sample budget that depends on the function evaluation cost and available resources. A simulation that takes several minutes may allow a few hundred iterations, as in the Daisy experiment of \sect{sec:realworld}. When function evaluations are A/B tests that take around a week, the evaluation budget may be limited to less than 50 iterations \citep[\textit{e.g.},][]{letham2019jmlr}. Generally in BO, there is a trade-off between the number of parameters one can optimize and the number of iterations that will be required, and in real problems one must select the number of parameters according to the evaluation budget. In that sense, there is no difference with linear embedding HDBO: $d_e$ should be set to the highest value that is supported by the available evaluation budget. With a budget of 50 iterations and $d_e=15$, it will be unlikely to get good model fit quickly enough to effectively optimize, so smaller values like $8$ or $10$ would be warranted. On the other hand, with the 500 iteration budget of the Daisy problem, one could set $d_e$ in the 15--20 range (the maximum supported by normal BO) to maximize $P_{\textrm{opt}}$.

Simulations and model cross-validation can be helpful for identifying the maximum number of parameters that can be effectively tuned for a particular evaluation budget, but there has been little work in this area. The nature of the dimensionality vs. iteration budget trade-off is important in all real BO problems, not just with linear embedding HDBO, so appropriate heuristics for this question is an important area of future work.

\section{Handling Black-Box Constraints in High-Dimensional Bayesian Optimization}\label{appendix:constrained}
In many applications of BO, in addition to the black-box objective $f$ there are black-box constraints $c_j$ and we seek to solve the optimization problem
\begin{align*}
    \textrm{minimize } &f(\vx)\\
    \textrm{subject to } & c_j(\vx) \leq 0, \quad j=1, \ldots, J,\\
    &\vx \in \mathcal{B}.
\end{align*}
In most settings the constraint functions $c_j$ are evaluated simultaneously with the objective $f$. Constraints are typically handled in BO by fitting a separate GP to each outcome (that is, to $f$ and to each $c_j$). The acquisition function is then modified to consider not only the objective value but also whether the constraints are likely to be satisfied \citep[\textit{e.g.},][]{gardner14}.

The extension of BO in an embedding to constrained BO is straightforward, so long as the same embedding is used for every outcome. A separate GP (in the case of ALEBO, using the Mahalanobis kernel) is fit to data from each outcome. Because the embedding is shared, predictions can be made for all of the outcomes at any point in the embedding. This allows us to evaluate and optimize an acquisition function for constrained BO in the embedding. Once a point is selected, it is projected up to the ambient space and evaluated on $f$ and each $c_j$ as usual. Random projections are especially well-suited for constrained BO because there is no harm in requiring the same projection for all outcomes, since it is a random projection anyway.

These same considerations apply to multi-objective optimization. Acquisition functions for multi-objective optimization can be directly applied to HDBO using linear embeddings in the same way that those for constrained optimization are used here.

\section{Additional Benchmark Experiment Results}\label{appendix:experiments}
Here we provide results from two additional benchmark problem (Hartmann6 $D$=100, and Hartmann6 random subspace $D$=1000), three additional methods (LineBO variants), and provide a study of the sensitivity of ALEBO performance to $d_e$ and $D$. We also provide implementation details for the experiments, and an extended discussion of the results from each experiment.

\subsection{Method Implementations and Experiment Setup}
The linear embedding methods (REMBO, HeSBO, and ALEBO) were all implemented using BoTorch, a framework for BO in PyTorch \citep{balandat19}, and so used the same acquisition functions and the same tooling for optimizing the acquisition function. Importantly, this means that all of the difference seen between the methods in the empirical results comes exclusively from the different models and embeddings. EI was the acquisition function for the Hartmann6 and Branin benchmarks, and NEI \citep{letham2019noisyei} was used to handle the constraints in the Gramacy problem. ALEBO and HeSBO were given a random initialization of 10 points, and REMBO was given a random initialization of 2 points for each of its 4 projections used within a run.

The remaining methods used reference implementations from their authors with default settings for the package: REMBO-$\phi k_\Psi$ and REMBO-$\gamma k_\Psi$\footnote{\url{github.com/mbinois/RRembo}}; EBO\footnote{\url{github.com/zi-w/Ensemble-Bayesian-Optimization}}; Add-GP-UCB \footnote{\url{github.com/dragonfly/dragonfly}, with option \texttt{acq="add\_ucb"}}; SMAC\footnote{\url{github.com/automl/SMAC3}, \texttt{SMAC4HPO} mode}; CMA-ES\footnote{\url{github.com/CMA-ES/pycma}}; and CoordinateLineBO, RandomLineBO, and DescentLineBO\footnote{\url{github.com/jkirschner42/LineBO}}. EBO requires an estimate of the best function value, and for each problem was given the true best function value. SMAC and CMA-ES require an initial point, and were given the point at the center of the ambient space box bounds. See the benchmark reproduction code at \repo{} for the exact calls used for each method.

The function evaluations for all problems were noiseless, so the stochasticity throughout the run and in the final value all comes from stochasticity in the methods themselves. For linear embedding methods the main sources of stochasticity are in generating the random projection matrix and in the random initialization.

\begin{figure}
    \centering
    \includegraphics{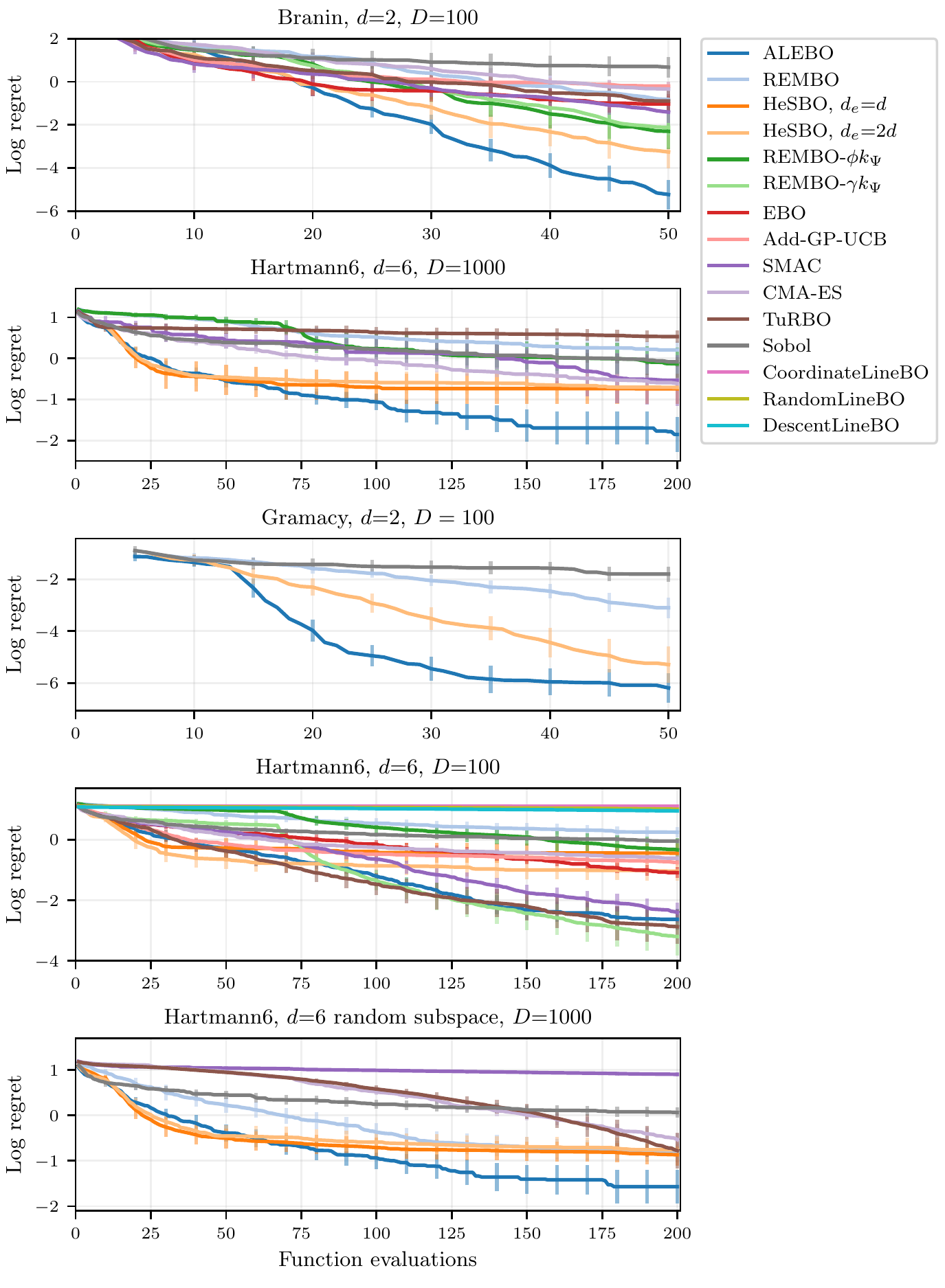}
    \caption{Log regret for the benchmark experiments of \fig{fig:synthetic_results}, plus Hartmann6 with $D$=100 and with a random (non-axis-aligned) subspace in $D$=1000. Each trace is the mean over 50 repeated runs, with errors bars showing two standard errors of the mean. ALEBO was a best-performing method on all problems; on Hartmann6 $D$=100 it tied with REMBO-$\gamma k_\Psi$, TuRBO, and SMAC as the best methods.}
    \label{fig:log_regrets}
\end{figure}

\subsection{Analysis of experimental results}

Fig. \ref{fig:log_regrets} provides a different view of the benchmark results of Fig. \ref{fig:synthetic_results}, showing log regret for each method, averaged over runs with error bars indicating two standard errors of the mean. This is evaluated by taking the value of the best point found so far, subtracting from that the optimal value for the problem, and then taking the log of that difference. The results are consistent with those seen in Fig. \ref{fig:synthetic_results}, and the standard errors show that ALEBO's improvement in average performance over the other methods is statistically significant. We now discuss some specific aspects of these experimental results.

\paragraph{Branin $D$=100} Starting from around iteration 20, ALEBO performed the best of all of the methods. The distribution of final iteration values shows that in one iteration the ALEBO embedding did not contain an optimum and so achieved a final value near 10. However, across all 50 runs nearly all achieved a value very close to the optimum, leading to the best average performance. Without the log transform (\fig{fig:synthetic_results}), SMAC and the additive GP methods were the next best performing.

The poor performance of HeSBO on this problem (particularly in Fig. \ref{fig:synthetic_results} without the log, where it is outperformed by all methods other than Sobol) can be attributed entirely to the embedding not containing an optimum. Recall that for this problem there are exactly three possible HeSBO embeddings, which are shown in Fig. \ref{fig:hesbo_embeddings}. As explained in \sect{appendix:hesbo}, the first embedding contains the optimum of 0.398, while the best value in the other embeddings are 0.925 and 17.18. Thus, if the BO were able to find the true optimum within each embedding with the budget of 50 function evaluations given in this experiment, the expected best value found by HeSBO would be:
\begin{equation*}
    0.398 P_{\textrm{opt}} + 0.925 \left(\frac{1-P_{\textrm{opt}}}{2}\right)  + 17.18 \left( \frac{1-P_{\textrm{opt}}}{2} \right).
\end{equation*}
This is the best average performance one can hope to achieve using the HeSBO embedding on this problem. Using (\ref{eq:hesbo_prob}) we can compute $P_{\textrm{opt}}$ for $d_e=4$ as 0.75, and it follows that the HeSBO expected best value is 2.56. This is nearly exactly the average best-value shown in Fig. \ref{fig:synthetic_results}. The poor performance of HeSBO is thus not related to BO, but comes entirely from the 12.5\% chance of generating an embedding whose optimal value is 17.18. The presence of these embeddings can be clearly seen in the distribution of final best values in Fig. \ref{fig:synthetic_results}.

\paragraph{Hartmann6 $D$=1000}
As noted in the main text, the additive kernel methods and REMBO-$\gamma k_\Psi$ could not scale up to the 1000 dimensional problem. SMAC also became very slow and was only run for 10 repeats (rather than 50) on the $D$=1000 problems. A nice property of linear embedding approaches is that the running time is not significantly impacted by the ambient dimensionality. Table \ref{tab:times} gives the average running time per iteration for the various benchmark methods (all run on the same 1.90GHz processor and allocated a single thread). Inferring the additional parameters in the Mahalanobis kernel and the added linear constraints make ALEBO slower than other linear embedding methods, but it is faster than the additive kernel methods (an order of magnitude faster than Add-GP-UCB), and at $D$=1000 is an order of magnitude faster than SMAC. The average of about 50s per iteration is short relative to the function evaluation time of typical resource-intensive BO applications.

\begin{table}
\centering
\caption{Average running time per iteration in seconds on the Hartmann6 problem, $D$=100 and $D$=1000.}
\begin{tabular}{l|c|c}
     & $D$=100 & $D$=1000  \\
     \hline
     ALEBO & 42.7 & 52.5\\
     REMBO & 1.6 & 1.9\\
     HeSBO, $d_e$=$d$ & 1.1 & 2.0\\
     HeSBO, $d_e$=$2d$ & 1.1 & 2.3\\
     REMBO-$\phi k_{\Psi}$ & 2.1 & 1.1\\
     REMBO-$\gamma k_{\Psi}$ & 7.2 & ---\\
     EBO & 69.6 & ---\\
     Add-GP-UCB & 995.0 & ---\\
     SMAC & 26.2 & 1137.9\\
     CMA-ES & 0.0 & 0.1\\
     Sobol & 0.1 & 0.8
\end{tabular}
\label{tab:times}
\end{table}

On both this problem and the $D$=100 version, REMBO performed worse than Sobol, despite there being a true linear subspace that satisfies the REMBO assumptions. The source of the poor performance is the poor representation of the function on the embedding illustrated in \fig{fig:rembo_illustration}. Correcting these issues as is done in ALEBO significantly improves the performance.

\paragraph{Hartmann6 $D$=100}
ALEBO, REMBO-$\gamma k_{\Psi}$, TuRBO, and SMAC were the best-performing methods on this problem. HeSBO and Add-GP-UCB both did very well early on, but then got stuck and did not progress significantly after about iteration 50. For HeSBO, this is likely because the performance is ultimately limited by the low probability of the embedding containing an optimum.

This problem was used to test three additional methods beyond those in \fig{fig:synthetic_results}: CoordinateLineBO, RandomLineBO, and DescentLineBO \citep{kirschner19}. These are recent methods developed for high-dimensional safe BO, in which one must optimize subject to safety constraints that certain bounds on the functions must not be violated. The performance of these methods can be seen in the fourth panel of \fig{fig:log_regrets}: all three LineBO variants perform much worse than Sobol, and show almost no reduction of log regret. This finding is consistent with the results of \citet{kirschner19}, who used the Hartmann6 $D$=20 problem as a benchmark problem. At $D$=20, they found that CoordinateLineBO required about 400 iterations to outperform random search, and even after 1200 iterations RandomLineBO and DescentLineBO did not perform better than random search. These methods are designed specifically for safe BO, which is a significantly harder problem than usual BO that has much worse scaling with dimensionality. The primary challenge for high-dimensional safe BO lies in optimizing the acquisition function, which is difficult even for relatively small numbers of parameters where there is no difficulty in optimizing the traditional BO acquisition function. The LineBO methods develop new techniques for acquisition function optimization, but do not consider difficulties with GP modeling in high dimensions, which is the main focus of HDBO work. LineBO methods perform very well on safe BO problems relative to other methods, but ultimately non-safe HDBO is not the problem that they were developed for, and so it is not surprising to see that they were not successful on this task.

\paragraph{Hartmann6 random subspace $D$=1000}
Linear embedding BO methods assume the existence of a true linear subspace, but do not assume anything about the orientation of that subspace and are generally invariant to rotation. Prior work on HDBO has typically focused on the axis-aligned (unused variables) problems that we used here, but we also include a non-axis-aligned problem. We generated a random true embedding by sampling a rotation matrix from the Haar distribution on the special orthogonal group SO$(D)$ \citep{stewart80}, and then taking the first $d$ rows to specify a projection matrix $\bm{T}$ from $D=1000$ down to $d=6$. This defines a non-axis-aligned true subspace, and we took the true low-dimensional function $f_d$ as the Hartmann6 function on this subspace. Bayesian optimization proceeded as with the other problems, and results for the linear subspace methods were similar to the axis-aligned $D$=1000 problem, except REMBO performed equal to HeSBO.

\subsection{Sensitivity of ALEBO to Embedding and Ambient Dimensions}\label{sec:subsec_ablation}

We study sensitivity of ALEBO optimization performance to the embedding dimension $d_e$ and the ambient dimension $D$ using the Branin function. To test dependence on $d_e$, for $D=100$ we ran 50 optimization runs for each of $d_e \in \{2, 3, 4, 5, 6, 7, 8\}$. To test dependence on $D$, for $d_e=4$ we ran 50 optimization runs for each of $D \in \{50, 100, 200, 500, 1000\}$. Note that the $d_e=4$ and $D=100$ case in each of these is exactly the optimization problem of \fig{fig:synthetic_results}.

\begin{figure}
\includegraphics{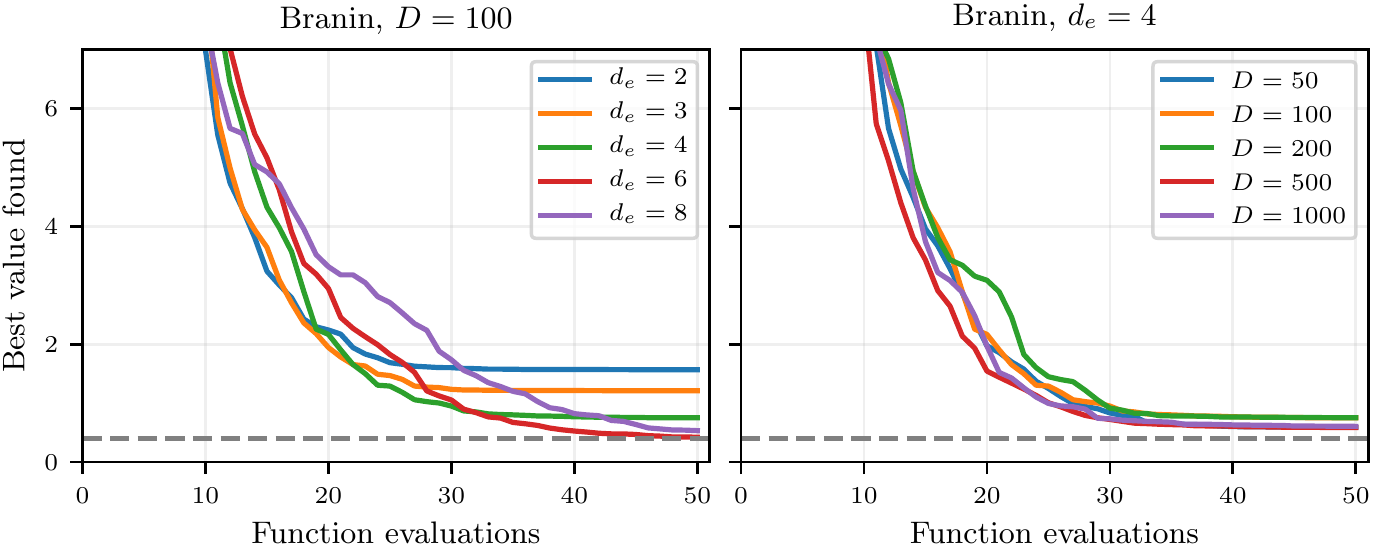}
\caption{ALEBO performance on the Branin problem, (\textit{Left}) as a function of embedding dimension $d_e$  and (\textit{Right}) as a function of ambient dimension $D$. Performance shown is the average of 50 repeated runs. Optimization performance is poor with $d_e=2$, but shows little sensitivity to $d_e$ for values greater than 2. Optimization performance shows little sensitivity with $D$, all the way up to $D=1000$.}
\label{fig:by_D_d_traces}
\end{figure}

The results of the optimizations are shown in Figs.~\ref{fig:by_D_d_traces} and \ref{fig:by_D_d}. For $d_e=d$, optimization performance was poor. From Fig. \ref{fig:lp_solns} we know this is because there is a low probability of the embedding containing an optimizer. Increasing $d_e$ increases that probability, but also increases the dimensionality of the embedding and thus reduces the sample efficiency of the BO in the embedding. This trade-off can be seen clearly in \fig{fig:by_D_d_traces}: with $d_e=2$ there is rapid improvement that then flattens out because of the lack of good solutions in the embedding, whereas for $d_e=8$ the initial iterations are worse but then it ultimately is able to find much better solutions. Even at $d_e=8$ the average best final value was better than that of any of the comparison methods in \fig{fig:synthetic_results}.

The ambient dimension $D$ will not directly impact the GP modeling in ALEBO, which depends only on $d_e$, however it will impact the probability the embedding contains an optimum as shown in \fig{fig:lp_solns_D}. Consistent with the strong ALEBO performance for the Hartmann6 $D$=1000 problem, we see here that even increasing $D$ to 1000 does not significantly alter optimization performance. Even at $D=1000$, ALEBO had better performance than the other benchmark methods had on $D=100$.

\begin{figure}
\includegraphics{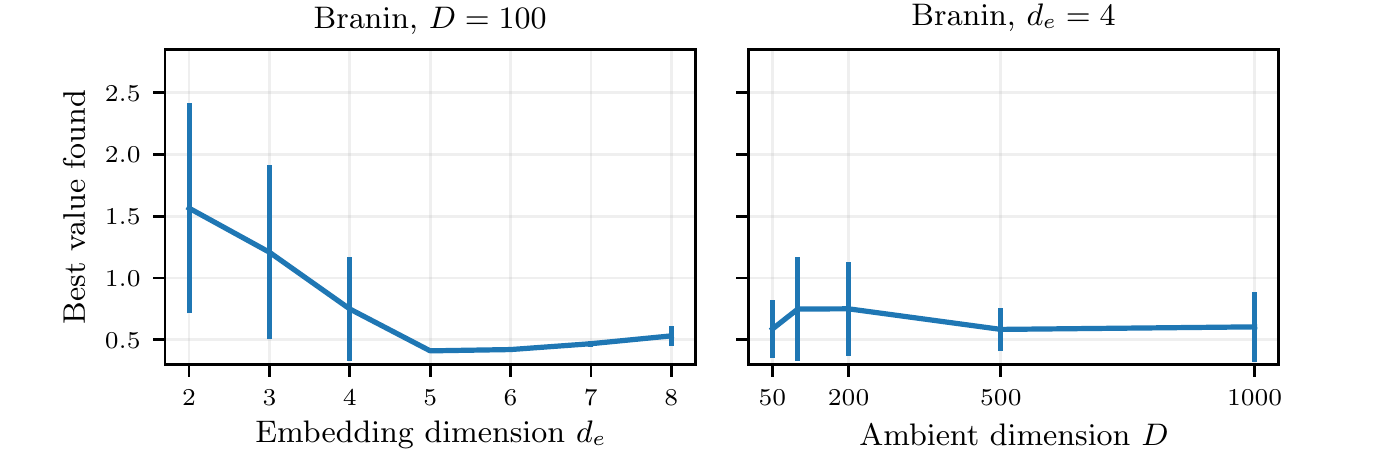}
\caption{Final best value for the Branin problem optimizations of \fig{fig:by_D_d_traces}, as mean with error bars showing two standard errors. With the exception of $d_e$ of 2 or 3, optimization performance was good across a wide range of values of $d_e$ and $D$.}
\label{fig:by_D_d}
\end{figure}

\subsection{Ablation Study}
We use an ablation study to better understand the impact of two of the new developments incorporated into ALEBO: the Mahalanobis kernel for improved modeling in the embedding, and the hypersphere sampling for increasing the probability that the embedding contains an optimum. We used the Branin $D=100$ problem for this study, with $d_e=4$ as in Fig. \ref{fig:synthetic_results}. For the ablation of the kernel, we replaced the Mahalanobis kernel with an ARD Matern 5/2 kernel. For the ablation of the sampling, we replaced hypersphere sampling with the random normal samples used by REMBO. Results are given in Fig. \ref{fig:ablation}. Removing either component significantly decreased BO performance, and removing the Mahalanobis kernel was especially detrimental. 

\begin{figure}
\centering
\includegraphics{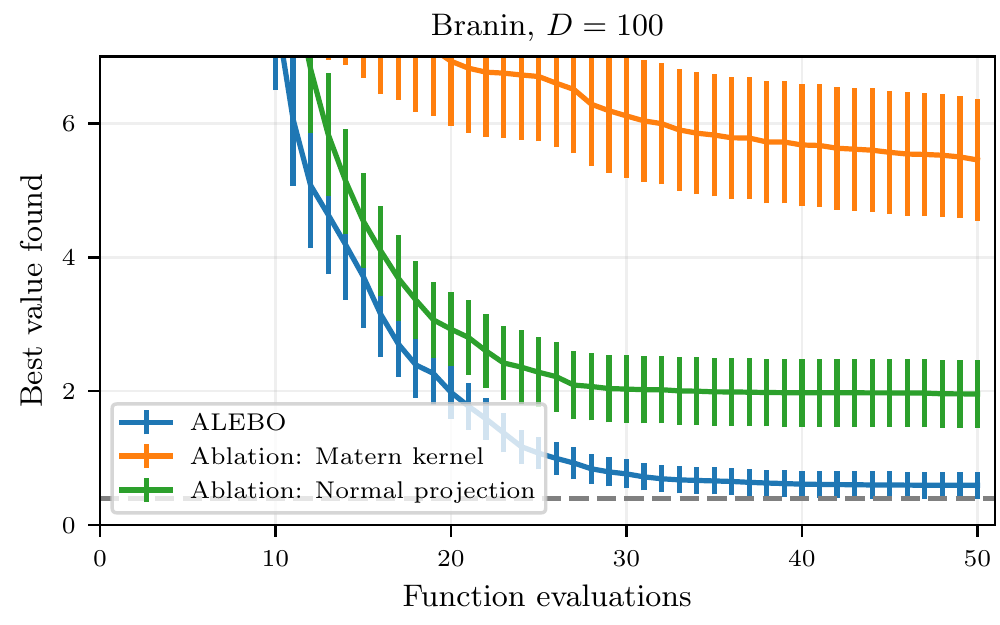}
\caption{Ablation study results comparing full ALEBO (same trace from Fig. \ref{fig:synthetic_results}) with: (1) Ablation of the Mahalanobis kernel, replacing it with an ARD Matern 5/2 kernel, and (2) Ablation of the hypersphere-sampled projection matrix, replacing it with an i.i.d. normal random matrix. Both components are necessary for the good HDBO performance of ALEBO, with the Mahalanobis kernel the most important factor.}
\label{fig:ablation}
\end{figure}

\section{Constrained Neural Architecture Search Problem}
\label{appendix:nas}
NAS-Bench-101\citep{nasbench} is a dataset of convolutional neural network (CNN) performance on the CIFAR-10 problem, produced for the purpose of reproducible research in neural architecture search. The search space is to design the cell for a CNN using a DAG with 7 nodes and up to 9 edges. The first node is input and the last node is output; the remaining five can be selected to be any one of the operations $3 \times 3$ convolution, $1 \times 1$ convolution, or $3 \times 3$ max-pool. The edges connect these operations to each other, and to the input and output. This space includes more than 400,000 unique models, each of which was evaluated on CIFAR-10 by training for 108 epochs on a TPU, and then testing on the test set. For each model, several metrics were computed, including the number of seconds required for training and the final test-set accuracy.

We parameterized this as a continuous HDBO problem by separately parameterizing the operations and edges. The operations were parameterized using one-hot encoding, which, with five selectable nodes and 3 options for each, produced 15 parameters. These were optimized in the continuous $[0, 1]$ space, and then the max for each set was taken as the ``hot" feature that specified which operation to use in the corresponding node. NAS-Bench-101 represents the edges using the upper-triangular $7 \times 7$ adjacency matrix, which has $\frac{(7 \cdot 6)} {2} = 21$ binary entries. These entries were similarly optimized in the continuous $[0, 1]$ space, and the adjacency matrix was created iteratively by adding entries in the rank order of their corresponding continuous-valued parameters, and finding the largest number of non-zero entries that can be added to still have no more than 9 edges after pruning portions not connected to input or output. The combination of the adjacency matrix parameters (21) and the one-hot-encoded operation parameters (15) produced a 36-dimensional optimization space.

The objective for the optimization was to maximize test-set accuracy, subject to a constraint on training time being less than 30 minutes. A large portion of the models in the NAS-Bench-101 modeling space have training times above 30 minutes, with the longest around 90 minutes. While test-set accuracy is proportional to training time \citep{nasbench}, our results in Fig. \ref{fig:nasbench} show that with HDBO it is possible to find well-performing models with short training times. TuRBO and CMA-ES were adapted to this constrained problem by return a poor objective value for infeasible points.

In the results of \fig{fig:nasbench}, HeSBO showed strong performance in early iterations, but quickly flattened out and ended up worse than random on average. By the final iteration, ALEBO performed significantly better than Sobol, CMA-ES, HeSBO, and REMBO. TuRBO performed worse on early iterations, but in later iterations had performance that was not significantly different from ALEBO.

\section{Locomotion Optimization Problem}
\begin{figure}
\includegraphics{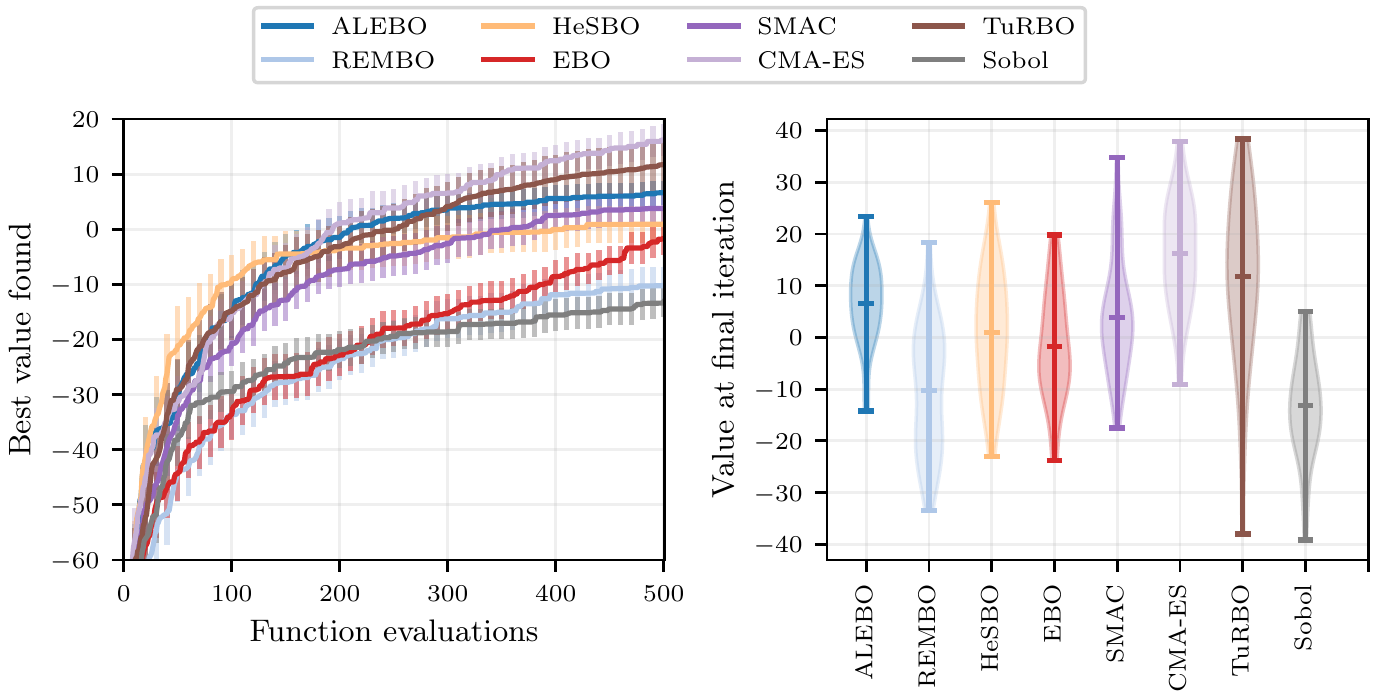}
\caption{Optimization performance on the hexapod locomotion task from \fig{fig:expdaisy}. \textit{(Left)} Each trace shows the best value by each iteration, averaged across repeated runs with error bars showing two standard errors. \textit{(Right)} The distribution of values at the final iteration. ALEBO performed best of the HDBO methods, but CMA-ES outperformed them all.}
\label{fig:daisy_2}
\end{figure}

\label{sec:appendix:daisy}
The task for the final set of experiments was to learn a gait policy for a simulated robot. As a controller, we use the Central Pattern Generator (CPG) from \cite{crespi2008online}. The goal in this task is for the robot to walk to a target location in a given amount of time, while reducing joint velocities, and average deviation from a desired height
\begin{equation}
    f(\vp) = C - ||\vx_{\textrm{final}} - \vx_{\textrm{goal}}|| - \sum^T_{t=0} ( w_1||\dot{\vq}_t|| - w_2|h_{\textrm{robot},t} - h_{\textrm{target}}|)\,,
\end{equation}
where $C = 10$, $w_1 = 0.005$, and $w_2 = 0.01$ are constants. 
$\vx_{\textrm{final}}$ is the location of the robot on a plane at the end of the episode, $\vx_{\textrm{goal}}$ is the target location, $\dot{\vq}_t$ are the joint velocities at time $t$ during the trajectory, $h_{\textrm{robot},t}$ is the height of the robot at time $t$, and $h_{\textrm{target}}$ is a target height. $T = 3000$ is the total length of the trajectory, leading to $30s$ of experiment. The objective function is evaluated at the end of the trajectory.

\fig{fig:daisy_2} shows the optimization performance over 50 repeated runs, which are the same results of \fig{fig:expdaisy} but including errors bars and the distribution of final best values. All of the methods have high variance in their final best value across runs. ALEBO has the lowest variance and thus the most robust performance. SMAC was able to find a good value in one run, but on average performed slightly worse than ALEBO. TuRBO performed better, but CMA-ES was the clear best-performing method on this problem.

\end{document}